\documentclass[11pt]{article}

\usepackage{moreverb,url}
\usepackage{natbib}

\usepackage[hidelinks]{hyperref}

\usepackage{times}
\usepackage{multicol}
\usepackage[LGR,T1]{fontenc}
\newcommand{\textgreek}[1]{\begingroup\fontencoding{LGR}\selectfont#1\endgroup}

\usepackage{amsmath}    % need for subequations
\usepackage{amsthm}    % need for subequations
\usepackage{graphicx}   % need for figures
\usepackage[font=footnotesize,labelfont=footnotesize]{caption}
\usepackage{verbatim}   % useful for program listings
\usepackage{color}      % use if color is used in text
\usepackage{mathtools}  % for a good looking colonequal

\usepackage{subfigure}  % use for side-by-side figures
\usepackage{hyperref} % use for hypertext links, including those to external documents and URLs
\usepackage{multirow}
\usepackage{rotating}
\usepackage{array}
\usepackage{xspace}
\usepackage{indentfirst}
\usepackage{amssymb}
\usepackage{float}
\usepackage{comment}
\usepackage{sidecap}
\usepackage{slashbox}

\usepackage[ruled,vlined,noend]{algorithm2e}
\usepackage{color}
\definecolor{darkblue}{rgb}{0.15,0.09,0.3}

\SetCommentSty{mycommfont}

\definecolor{darkred}{rgb}{.3333,0.0,0.0}

\SetAlFnt{\small}
\SetAlCapFnt{\normalsize}
\SetAlCapNameFnt{\normalsize}
\usepackage{algorithmic}
\algsetup{linenosize=\tiny}

\newcommand\BibTeX{{\rmfamily B\kern-.05em \textsc{i\kern-.025em b}\kern-.08em
T\kern-.1667em\lower.7ex\hbox{E}\kern-.125emX}}

\let\emptyset\varnothing

\usepackage{graphicx} % some of the .figs have text labels that are rotated

\newcommand*{\medcap}{\mathbin{\scalebox{1.5}{\ensuremath{\cap}}}}%

\newcommand{\set}[1]{\ensuremath{{\rm #1}}\xspace}
\newcommand{\lang}[1]{\ensuremath{{\rm #1}}\xspace}
\newcommand{\func}[1]{\ensuremath{{\rm #1}}\xspace}
\newcommand{\defn}{\coloneqq}
\newcommand{\ks}{\ast}
\newcommand{\emptyseq}{\varepsilon}

\renewcommand{\emptyset}{\varnothing}
\newcommand{\Vinit}{V_{\scriptscriptstyle 0}}
\newcommand{\Vu}{V_{u}}
\newcommand{\Vy}{V_{y}}

\newcommand{\Vterm}{V_{\rm term}}
\newcommand{\Vtermp}{V'_{\rm term}}
\newcommand{\Vgoal}{V_{\rm goal}}
\newcommand{\Vgoalp}{V'_{\rm goal}}

\newcommand{\card}[1]{\ensuremath{|{#1}|}\xspace}

\newcommand{\Pterm}{P_{\rm term}}
\newcommand{\Qterm}{Q_{\rm term}}
\newcommand{\Rterm}{R_{\rm term}}

\newcommand{\mysq}{\scalebox{0.35}{$\square$}}
\newcommand{\act}[1]{\ensuremath{{#1}^{{}^{{}_\bullet}}\xspace}}
\newcommand{\obs}[1]{\ensuremath{{#1}^{{}^{{}_{\mysq}}}}\xspace}

\def\markatright#1{\leavevmode\unskip\nobreak\quad\hspace*{\fill}{#1}}

\newcommand{\edge}[3]{{#1}\xrightarrow{#2}{#3}}
\newcommand{\gobble}[1]{}

\newcommand{\eseq}{e_1 \cdots e_k}
\newcommand{\epseq}{e'_1 \cdots e'_m}
\newcommand{\pow}[1]{\ensuremath{{2}^{#1}}\xspace}
\newcommand{\real}{\ensuremath{\mathbb R}\xspace}

\newcommand{\natzero}{\ensuremath{\mathbb N_0}\xspace}

\newcommand{\langof}[1]{\ensuremath{\mathcal{L}({#1})}\xspace}

\newcommand{\defeq}{\ensuremath{\coloneqq}}

\makeatletter
\let\example\@undefined
\let\endexample\@undefined    
\makeatother

\makeatletter
\let\theorem\@undefined
\let\endtheorem\@undefined    
\let\c@theorem\@undefined
\makeatother
\makeatletter
\let\lemma\@undefined
\let\endlemma\@undefined    
\let\c@lemma\@undefined
\makeatother
\newtheorem{definition}{Definition}[section]
\newtheorem{theorem}{Theorem}[section]
\newtheorem{corollary}{Corollary}[theorem]
\newtheorem{lemma}[theorem]{Lemma}

\theoremstyle{definition} % Examples are too long to have in italics
\newtheorem{example}{Example}[section]

\begin{document}

\title{Toward a language-theoretic foundation\\ for planning and filtering}

\author{Fatemeh Zahra Saberifar, Shervin Ghasemlou,\\ Dylan A. Shell, and Jason~M.~O'Kane}

\date{July 6, 2018}

\maketitle

%% \affiliation{\affilnum{1}Department of Mathematics and Computer Science, Amirkabir University of Technology, Tehran, Iran,
%% \affilnum{2}Department of Computer Science and Engineering, University of South Carolina, Columbia SC, USA,
%% \affilnum{3}Department of Computer Science and Engineering, Texas A\&M University}
%% \corrauth{Author $x$, Address of University A, XY, 11111 USA}
%% \email{corresponding@abc.edu}

\begin{abstract}
  %least 70 and at most 150 words.
  We address problems underlying the algorithmic question of automating the
  co-design of robot hardware in tandem with its apposite software.
  Specifically, we consider the impact that degradations of a robot's sensor
  and actuation suites may have on the ability of that robot to complete its
  tasks.
  We introduce a new formal structure that generalizes and consolidates a
  variety of well-known structures including many forms of plans, planning
  problems, and filters, into a single data structure called a procrustean
  graph, and give these graph structures semantics in terms of ideas based in
  formal language theory.
  We describe a collection of operations on procrustean graphs (both
  semantics-preserving and semantics-mutating), and show how a family of
  questions about the destructiveness of a change to the robot hardware can be
  answered by applying these operations.
  We also highlight the connections between this new approach and
  existing threads of research, including combinatorial filtering, Erdmann's
  strategy complexes, and hybrid automata.
  %
  %A detailed case study, aided by an implementation of our algorithms, is
  %described.
  %
\end{abstract}

\section{Introduction}\label{sec:intro}

The process of designing effective autonomous robots---spanning the selection
of sensors, actuators, and computational resources along with software to
govern that hardware---is a messy endeavor.
There appears to be little hope of fully automating this process, at least in
the short term.  There would, however, be significant value in \emph{design
tools} for roboticists that can manipulate partial or tentative designs, in
interaction with a human co-designer.
For example, one might imagine algorithms that answer questions about the
relationship between a robot's sensors and actuators and that robot's ability
to complete a given task.

A crucial requirement for this kind of automation is a general formal model
that can describe, in a precise way, a robot's sensing and actuation
capabilities in the context of its interaction with an environment.
To that end, this paper lays theoretical groundwork for reasoning about sensors
and actuators and their associated estimation and planning processes.  The
underlying goal is to strengthen the link between idealized models and
practical\,---that is, imperfect, imprecise, and limited---\,realizations of
those idealized models in actual, available hardware.

To motivate these questions more concretely, consider the pair of scenarios
that follows.
%They emphasize the paper's focus on pragmatic concerns, in spite of its
%apparent theoretical flavor.

\begin{example}
  Your robot is stationed on a distant planet and, though fully operable
  initially, has recently encountered a problem. It appears that debris has
  become affixed to one of the sensors.  Should operations be altered by taking
  more conservative paths around obstacles because the robot's position
  estimates now involve greater error than previously?  Or has the mission been
  entirely compromised?  Assuming that the debris cannot be dislodged, what
  tasks are still feasible?
\end{example}

\begin{example}
  You lead an R\&D team who have built and tested a successful prototype robot,
  which performs cosmetic services (e.g.,~manicures, pedicures, facials,
  hair-weaves, etc.) efficiently and safely.
  %Market research shows demand at the price you can meet, indicating that it's
  %time to move to production of the robot.  You find a country with liberal
  %markets and rock-bottom labor costs, and choose to overlook the oppressive
  %government's human-rights violations.
  % A production line is planned; construction of the
  % factory begins; investors are licking their lips, and then disaster strikes!
  %
  Then\dots disaster!  You discover that the sensor provided to your factory in
  bulk (say $S_1$), differs from the device ($S_0$) supplied by the same
  manufacturer to the team who built and tested the prototype.
  %You rack your %brain with these questions:
  A successful redesign of the robot might require answers to these kinds of
  questions: Can $S_1$ be used directly as a plug-and-play replacement for
  $S_0$?  If not, can we adjust some software parameters to make it work?
  Which parameters and what should the adjustments be? If $S_1$ necessarily incurs
  a loss in performance, how can this be understood---perhaps only the
  hair-styling functionality is affected?  Supposing we can procure $S_0$ at
  greater cost through another vendor, is this worth doing?
\end{example}

\noindent Underlying these scenarios is the problem of how to ascertain whether
or not a particular modification to one's model of a robot is destructive for a given
task.  In this paper, we formalize this question, providing theoretical
foundations as well as algorithms to address problems of this type.
This can be seen visually in Figure~\ref{fig:lasers'n'wheels}.

\begin{figure}
  \centering
  \includegraphics[scale=0.6]{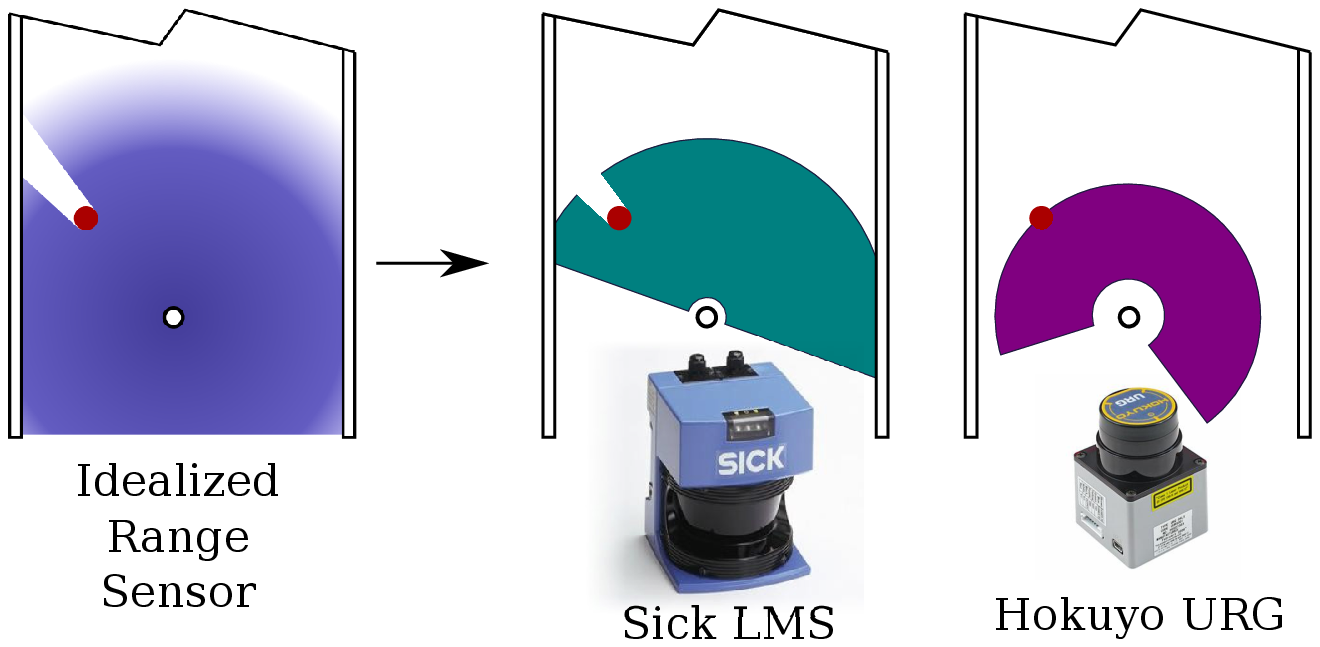}\\
  \vspace{10pt}
  \hspace*{-24pt}\includegraphics[scale=0.6]{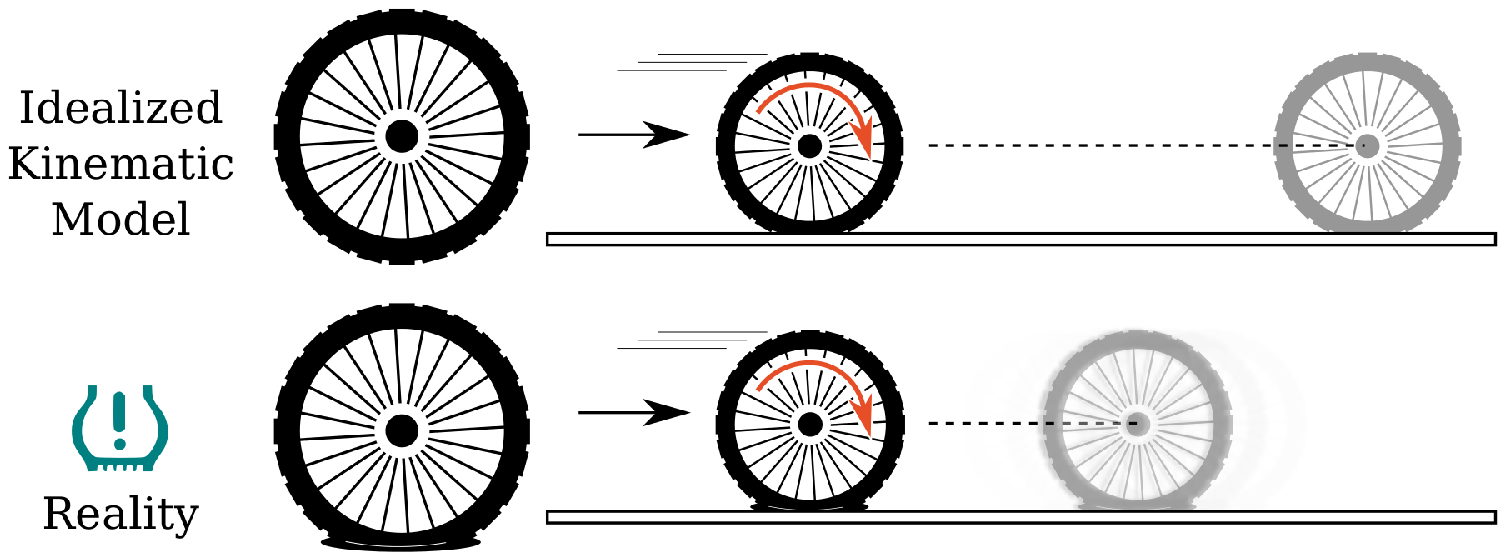}
  \caption{
    Given a specification of robot capabilities---encompassing both sensors
    (top) and actuators (bottom)---and a task, which changes to those sensors
    or actuators make the task infeasible?  This question is fundamental
    because it addresses the link between idealized models and practical
    (imperfect, imprecise, and limited) realizations of those idealized models
    in hardware, and because its answers can can lead to insight about weakest
    robots that suffice for a given task. % and, even, the converse: the most powerful task a given robot might perform.
    \label{fig:lasers'n'wheels}
  }
\end{figure}

This paper makes several new contributions.
\begin{enumerate}
  \item We introduce, in Section~\ref{sec:ilang}, the notion of an
  \emph{interaction language}, which models the interactions between a robot
  and its environment using the theory of formal languages.  This approach
  unifies several previously distinct conceptual classes of object.  

  \item We contribute, in Section~\ref{sec:pgraph}, a general representation
  called a \emph{procrustean graph}\footnotemark{} for interaction languages.
  This representation is constructive, in that it can be used to instantiate a
  data-structure from which various questions can be posed and addressed
  concretely.
  \footnotetext{Named for Procrustes (\textgreek{Prokro\'{u}stes}), son of
  Poseidon, who, according to myth, took the one-size-fits-all concept to
  extremes.} 
  
  \item We show, in Section~\ref{sec:maps}, how to model degradations to
  sensing and action capabilities in this framework as \emph{label maps}.  We
  address the question of deciding whether a label map is \emph{destructive},
  in the sense of preventing the achievement of a previously-attainable goal,
  for both filtering and planning problems, in Sections~\ref{sec:filters} and
  \ref{sec:plans} respectively.  We also prove that the broader question of
  finding a non-destructive label map that is, in a certain sense, maximal, is
  NP-hard.

\end{enumerate}
The d\'{e}nouement of the paper includes a review of related work interleaved
with a discussion of the outlook for continued progress
(Section~\ref{sec:related}) and some concluding remarks
(Section~\ref{sec:conclusion}).

Preliminary versions of this work appeared in 2016 at RSS~\citep{SabGha+16} and WAFR~\citep{GhaSab+16}.

\section{Actions, observations, and interaction languages}\label{sec:ilang}
% \subsection{Basic Definitions}

We begin with some basic definitions for modeling the interaction between an
agent or robot and its environment.  The robot executes \emph{actions} drawn
from a non-empty \emph{action space} $\set{U}$; the environment yields
\emph{observations} drawn from a non-empty \emph{observation space} $\set{Y}$.
We assume that $\set{U} \cap \set{Y} = \emptyset$.
Neither need necessarily be finite.
Noting the duality between actions and observations---an observation can be
viewed merely as an `action by nature'---we treat actions and observations as
specific subtypes of a more general class of events.

\begin{definition}[event]
  An \emph{event} is an action or an observation.  The \emph{event space} is
  $\set{E} \defn \set{U} \cup \set{Y}$.
\end{definition}

\begin{definition}[event sequence]
  An \emph{event sequence} over $E$ is a finite sequence of events $e_1 \cdots
  e_m$ drawn from $E$.
  An event sequence is called \emph{action-first} if $e_1 \in \set{U}$, or
  \emph{observation-first} if $e_1 \in \set{Y}$.
  Likewise, an event sequence is called \emph{action-terminal} if $e_m \in
  \set{U}$, or \emph{observation-terminal} if $e_m \in \set{Y}$.
\end{definition}

\begin{definition}[successor]
  For two event sequences $s_1$ and $s_2$ over the same event set $\set{E}$, we
  say that $s_2$ is a \emph{successor} of $s_1$, if there exists some
  $e\in\set{E}$ such that $s_2 = s_1 e$.
\end{definition}

In what follows, we describe sets of event sequences using standard notation
for regular expressions: Concatenation (represented implicitly using
juxtaposition), alternation (using the binary $+$ operator), the empty sequence
(the $\emptyseq$ symbol), and the Kleene star (the unary $\ks$ operator).

\begin{definition}[interaction language]\label{def:ilang}
  An \emph{interaction language} $\lang{L}$ over an event space~$\set{E}$ is a
  set of event sequences which is either
  \begin{enumerate}
    \item a subset of $(\set{U}\set{Y})^\ks(\set{U} + \{\emptyseq\})$, or
    \item a subset of $(\set{Y}\set{U})^\ks(\set{Y} + \{\emptyseq\})$,
  \end{enumerate}
  and which is closed under prefix.
\end{definition}
The intuition behind interaction languages is that they describe sequences of
events which alternate between action and observation.  The definition admits
two distinct types of interaction languages: those whose members begin with
actions (hereafter, \emph{action-first languages}) and those whose members
begin with observations (\emph{observation-first languages}).

Interaction languages encode an interaction between an agent or robot (which
selects actions) and its environment (which dictates the observations made by
the robot). The definition is intentionally ecumenical in regard to the nature
of that interaction, because we intend this definition to serve as a starting
point for more specific structures which, once specific context and semantics
are added, lead to special cases that represent particular (and familiar)
objects involving planning, estimation, and the like. 

The prefix-closure requirement in Definition~\ref{def:ilang} ensures that for
all event sequences $e_1 e_2\cdots e_m \in \lang{L}$, every subsequence $e_1
\cdots e_k$ with $k < m$ is also in $\lang{L}$.
If a language expresses properties of some structured interaction, then some
event sequences are excluded from that language.  In such cases then the prefix
condition captures the idea that part of the way through a sequence, or even in
a sequence stopped short, anterior structure is present.

% \subsection{Interaction Language Examples}

The examples that follow illustrate a few different kinds of interaction languages.

\begin{example}[Filters]\label{ex:filter}
  Filtering, in very broad terms, refers to any process by which observations
  are processed to produce specified outputs.  That is, a specification of a
  filter tells us, for any plausible history of observations that an agent
  might have made, what the correct output from the filter should be.
  Filters are, of course, objects of intense and sustained interest within the
  robotics community.
  
  In light of Definition~\ref{def:ilang}, we can describe a filter as an
  action-first interaction language $L$, in which the filter's outputs are
  modeled as actions ${\rm emit}_{x}$ for various outputs $x$.
  %and for which
  %each observation-terminal event sequence in $L$ has exactly one successor in
  %$L$.

  % TODO: Toy example: Duplex light switch with global shutoff
  % - Observations:
  %   - a = Switch 1 flipped up
  %   - b = Switch 1 flipped down
  %   - c = Switch 2 flipped up
  %   - d = Switch 2 flipped down
  %   - e = Global shutoff triggered
  % - Actions:
  %   - emit0 - light is off
  %   - emit1 - light is on
  % - Initial condition:
  %   - light is off
  %   - both switches down
  % - Language:
  %   - Can't see b before seeing an a.
  %   - After a, can't have another a before seeing b.
  %   - etc.
  %   - Action-terminal sequences must end with the 'correct' output.
  %       + If there's an e, then emit0.
  %       + Otherwise, count number of toggles: even -> emit0, odd -> emit1

  %  \begin{figure}[h]
  %    \centering
  %    \includegraphics[scale=0.30]{figures/switches}
  %    \caption{A setup with a duplex light switch ($a, \bar{a}$, and $b, \bar{b}$) with a permanent kill action $c$, described by language ....}
  %  \end{figure}

\end{example}

\begin{example}[Schoppers's universal plans]
\label{ex:schoppers}
  For observable domains, a universal plan~\citep{schoppers87} is a
  specification of an appropriate action for each circumstance that an agent
  might find itself in.
  This kind of model can be expressed as an interaction language $L$ in which
  each observation corresponds to a world state, and for each observation $y
  \in Y$, every event sequence ending in $y$ has exactly one successor, and
  that these successors are all formed by appending a single unique action $u$.
  The intuition is that this unique $u$ indicates the action that should be
  taken when the robot is in the state corresponding to $y$.
\end{example}

\begin{example}[Erdmann-Mason-Goldberg-Taylor plans]
\label{ex:emgt1}
  Several classic papers\\
  \citep{ErdMas88,mason88planningseq,Gol93} find
  policies for manipulating objects in sensorless (or nearly sensorless)
  conditions.  The problems are usually posed in terms of a polygonal
  description of a part; the solutions to such problems are sequences of
  actions.  
  Such plans can be expressed as interaction languages containing all event
  sequences in which the actions (e.g., a squeeze-grasp or a tray tilt at a
  particular orientation) in each sequence guarantee a known final orientation
  of the part regardless of its unknown initial orientation.
  In the event sequences of the interaction language, these actions are
  interleaved with with a special $\eta$ which constitutes the sole
  element in $Y$, acting as dummy `no observation' token.
\end{example}

\begin{example}[Counting amidst beams]\label{ex:notreg}
  As another example, consider a system in which an unknown number of agents
  moves through a known network of rooms.  Their movements are observed by
  discrete beam sensors that detect the passage of an agent from one room to
  another, but not the identity of that agent.  Actions allow barriers between the
  rooms to be opened or closed.  (Similar problems were addressed by
  \citet{EriYu+14} and \citet{GieBob+14}.)
  The evolution of this kind of system can be modeled as an interaction
  language whose event sequences are those that correspond to valid traces of
  this system.

  Figure~\ref{fig:rooms} shows an example, in which $a$ observations indicate
  an agent moving from $r_1$ to $r_2$, and $b$ observations indicate an agent
  moving from $r_2$ to $r_1$.  Some unknown number of agents begins in $r_1$,
  whereas $r_2$ is initially empty.
  Interestingly, even for this very simple case, the interaction language is
  not a regular language.
  To see this, note that the number of $a$ observations must be no less
  than the number of $b$ observations in any event sequence that
  occurs in this system---no agent can leave $r_2$ if that room is
  already empty---and no finite-state automaton can do this kind of `counting'
  for arbitrarily many agents.

  \begin{figure}[t]
    \centering
    \includegraphics[trim={0 0 2.2cm 0},clip,scale=1.33]{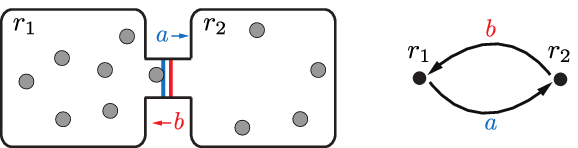}
    \caption{ A simple environment with two regions and beam sensors in the
    corridor connecting them; a grey body is moving from region $r_1$
    to region $r_2$ (drawn after~\citet{EriYu+14}, but simplified). 
    For an unknown number of grey agents, all of which start in $r_1$.
    The interaction language for this system is not a regular language.
    \label{fig:rooms}}
  \end{figure}
\end{example}

% \subsection{Pairs of Languages}
Next, we introduce some definitions for reasoning about relationships between
pairs of interaction languages, in terms of the event sequences that are shared
between them.  In Section~\ref{sec:plans}, we model both planning problems and
plans themselves via interaction languages.  The next definitions will be
helpful for formalizing the relationships between those two languages.

At the simplest level, we recall the distinction between
action-first and observation-first interaction languages.

\begin{definition}[akin]
  Two interaction languages $\lang{L_A}$ and $\lang{L_B}$, both over the same
  set of events, are \emph{akin} if they are both action-first languages, or
  they are both observation-first languages.
\end{definition}

We can also consider the set of event sequences shared between a pair of akin
interaction languages.

\begin{definition}[joint event sequence]\label{def:joint}
  Given two interaction languages $\lang{L_A}$ and $\lang{L_B}$ that are akin, an
  event sequence $s$ is a \emph{joint event sequence} if $s \in \lang{L_A}$ and
  $s \in \lang{L_B}$.
\end{definition}

As an aside, we note that the structure required of interaction languages is
preserved when we consider only the set of joint event sequences for a pair of
languages.

\begin{theorem}[Joint event sequences form an interaction language]
  For any two interaction languages $L_A$ and $L_B$ that are akin, the set $L_A \cap
  L_B$ of their joint event sequences is itself an interaction language. 
\end{theorem}
\begin{proof}
  Follows directly from Definitions~\ref{def:ilang} and \ref{def:joint}.
\end{proof}

Of particular interest in the context of planning will be pairs of languages
for which there is some bound on the longest joint event sequence.  The next
definition makes that intuition more precise.

\begin{definition}[finite on]\label{def:finite}
  Given two akin event languages $\lang{L_A}$ and $\lang{L_B}$, if there exists
  an integer $k$ that bounds the length of every joint event sequence of
  $\lang{L_A}$ and $\lang{L_B}$, we say $\lang{L_A}$ is \emph{finite on}
  $\lang{L_B}$.
\end{definition}

Note that `finite on' is a symmetric relation, though the way it is written
does not immediately emphasize this fact. Some caution is likely warranted as the
definitions have made a subtle departure from standard language theory.
In particular, finite on does not require that the either of the interaction
languages be finite sets, but only that there exist some finite bound on the
lengths of their joint event sequences.  In fact, the set of joint sequences
may form an infinite set since $\set{U}$ or $\set{Y}$ need not be finite;
Definition~\ref{def:finite} requires only a bound on the length of the
sequences.

Finally, we consider a notion of `compatibility' between two interaction
language.

\begin{definition}[safety]\label{def:safe}
  Given two event languages $\lang{L_A}$ and $\lang{L_B}$, both akin to one
  another, $\lang{L_A}$ is \emph{safe on} $\lang{L_B}$ if, for every joint
  event sequence $s$, the following holds:
  \begin{enumerate}
    \item if $s$ is observation-terminal, then for every successor $s'$ of $s$,
        \[s'\in\lang{L_A} \implies s'\in\lang{L_B};\]
    \item if $s$ is action-terminal, then for every successor $s'$ of $s$,
        \[s'\in\lang{L_B} \implies s'\in\lang{L_A}.\]
  \end{enumerate}
\end{definition}

To understand the intuition, imagine a joint event sequence constructed one event at a
time, with actions selected via the successors in $L_A$ of the current event sequence and observations
selected via the successors of $L_B$.  When the next event should be an action,
Definition~\ref{def:safe} requires that $L_B$ must be `ready' (in the sense of
containing at least one suitable event sequence) for any action that might be
selected from the successor event sequences in $L_A$.  Likewise, when the next
event should be an observation, $L_A$ must be ready for any observation that
might be selected from the successor event sequences in $L_B$.

Though the symmetry in the definition is perhaps aesthetically pleasing,
one should not be misled: safety of $L_A$ on $L_B$ does not imply that $L_B$ is
safe on $L_A$.  Moreover, safety is not transitive.  (Note that appearing on
the left differs from appearing on the right, as the quantifiers shift.)
However, safety is indeed is reflexive ($L_A$ is always safe on~$L_A$).

\section{Procrustean graphs and set labels}\label{sec:pgraph}
%\subsection{Basic Definitions}
\subsection{Procrustean graphs}

The languages and other definitions in the preceding section express the fact
that interactions may possess structure. Though formal, they are not a directly
useful construct for algorithmic manipulation nor for reasoning about causality
in the aspects involved. To automate (or help automate) design-time processes,
we introduce a new representation called a procrustean graph for a broad class
of interaction languages, based on graphs with transitions labeled by sets.

\begin{definition}[p-graph]\label{def:pgraph}
  A \emph{procrustean graph (p-graph)} over event space $\set{E}$
  is a finite edge-labeled bipartite directed
  graph in which
  \begin{enumerate}
    \item the finite vertex set, of which each member is called a \emph{state}, can be
    partitioned into two disjoint parts, called the \emph{action vertices}
    $\set{\Vu}$ and the \emph{observation vertices} $\set{\Vy}$, with $\set{V} = \set{\Vu} \cup \set{\Vy}$,

    \item each edge $e$ originating at an action vertex is labeled with a set
    of actions $\set{U(e)}\subseteq\set{U}\subset\set{E}$ and leads to an observation vertex,

    \item each edge $e$ originating at an observation vertex is labeled with a
    set of observations $\set{Y(e)}\subseteq\set{Y}\subset\set{E}$ and leads to an action vertex, and

    \item a non-empty set of states $\set{\Vinit}$ are designated as \emph{initial
    states}, which may be either exclusively action states ($\set{\Vinit} \subseteq
    \set{\Vu}$) or exclusively observation states ($\set{\Vinit} \subseteq \set{\Vy}$).
  \end{enumerate}

\end{definition}

\noindent A small example, intended to illustrate the basic intuition of the
definition, follows.

\begin{example}[wheels, walls, and wells]
  Figure~\ref{fig:pgraph-example} show a p-graph that
  models a Roomba-like robot that uses
  single-bit wall and cliff sensors to navigate through an environment.
  Action states are shown as unshaded squares; observation states are shaded
  circles.
  Action labels are subsets of $[0,500]\times[0,500]$, of which each element
  specifies velocities for the robot's left and right drive wheels, expressed
  in mm/s.  Observations are bit strings of length 2, in which the first bit is
  the output of the wall sensor, and the second bit is the output of the cliff
  sensor.
  \begin{figure}[t]
    \centering
    ~~\includegraphics[scale=0.4]{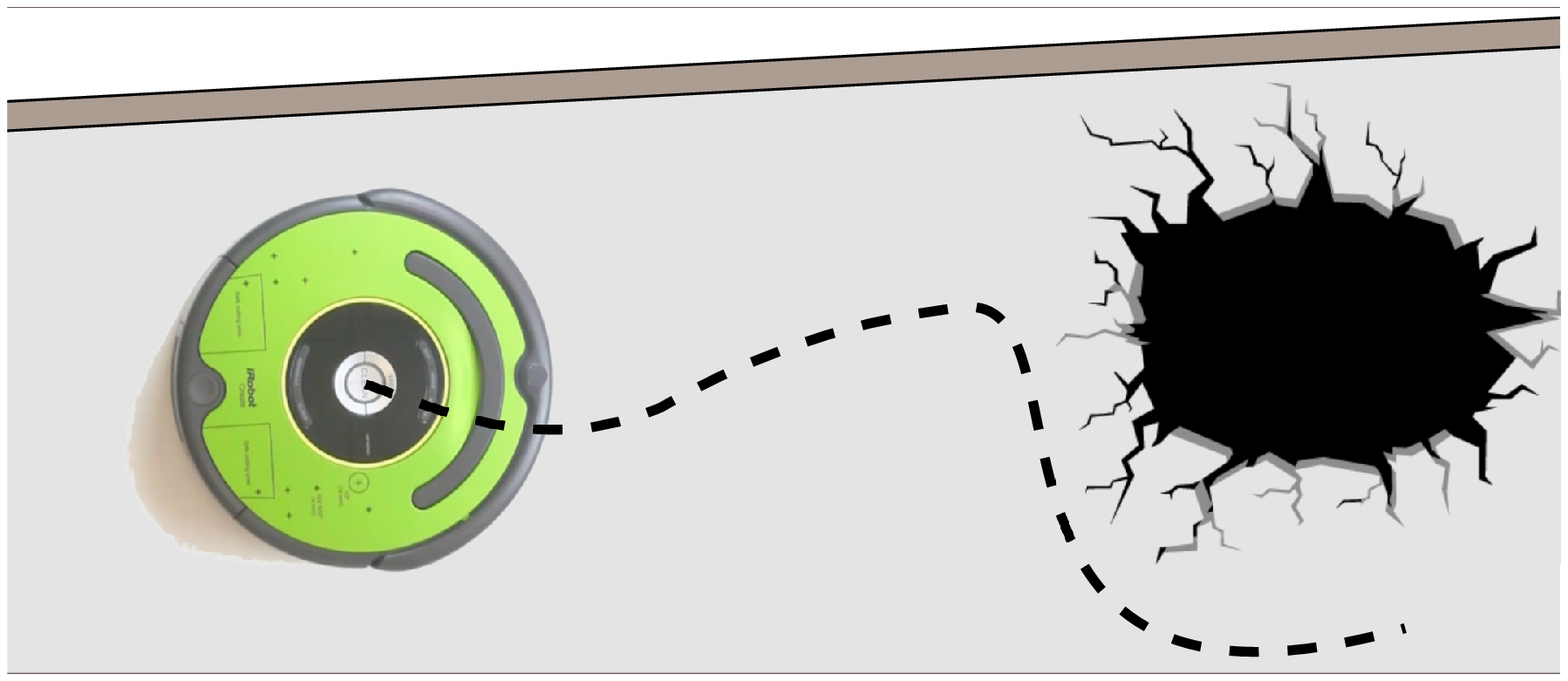}\hfill\phantom{}

    \vspace*{-110pt}\phantom{}\hfill\includegraphics[scale=0.52]{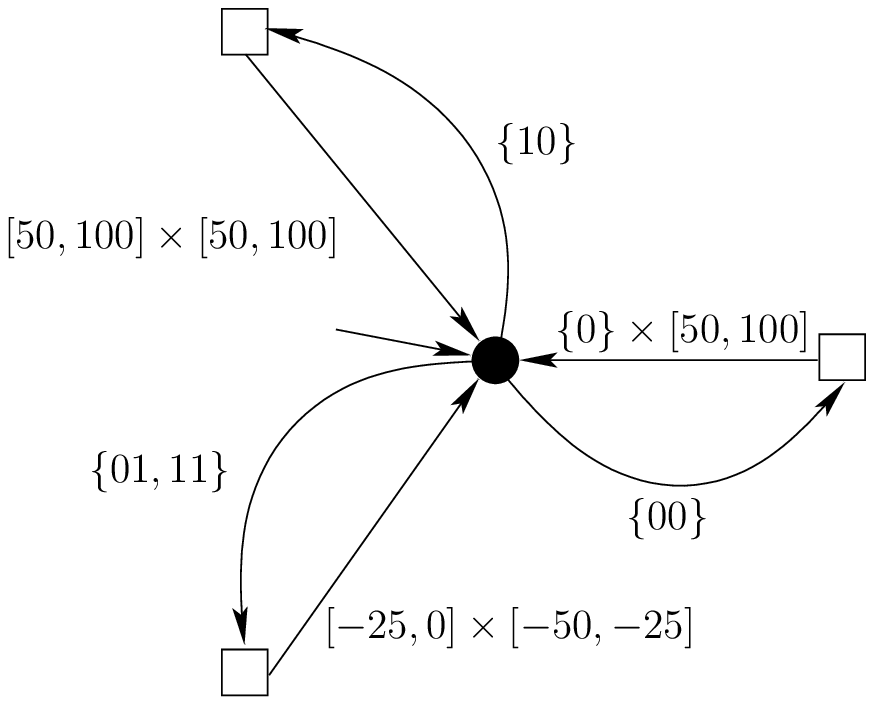}~~
    \caption{[left] A differential drive robot with sensors for obstacles, both
    positive (walls) and negative (holes). [right] An example p-graph that
    models behavior in which the robot follows a wall while avoiding negative
    obstacles. This graph, and those that follow, have solid circles to
    represent elements of $\Vu$, and empty squares for $\Vy$. The 
    arcs are labeled with sets; those that leave the central vertex have two
    digits, the first digit is `$1$' iff the wall is detected by the IR sensor
    on the left-hand side; the second digit is `$1$' iff the downward pointing IR
    sensor detects a cliff. The actions, on the edges leaving squares, represent
    sets of left and right wheel velocities, respectively.
    \label{fig:pgraph-example}}
  \end{figure}
\end{example}

% Note, I've not said anything about finiteness of states yet, so this is a careful
% statement.

P-graphs bear a close relationship
to interaction languages---they describe sets comprised of sequences of
actions and observations that alternate. The intention is for p-graphs to serve
as concrete data structures for representing interaction languages. This helps realize 
the paper's objective, which is to treat p-graphs, in a general sense, as
first-class objects, suitable for manipulation by automated means.  

Before formalizing the details of the connection between p-graphs and
interaction languages, we make a minor detour to show that p-graphs are
sufficiently rich to describe things that have been of broad interest to
roboticists for a long time.

\begin{example}[Combinatorial filters]
\label{ex:combfilter}
  Recall from Example~\ref{ex:filter} that filtering problems can be cast in
  terms of interaction languages in which the filter outputs are encoded as
  actions of the form ${\rm emit}_x$.
  A particular class of filters that is well suited to representation as p-graphs
  are the \emph{combinatorial filters}.  As formalized by
  \citet{lavalle12sensing}, combinatorial filters are discrete expressions of
  estimation problems. 
  More precisely, combinatorial filters are finite-state transition systems in
  which each state has a specific output associated with it.
  Such filters can be cast as p-graphs by having observations and observation
  transitions exactly as in the filter, but with action vertices having only a
  single out-edge that is labeled with a singleton set bearing the output
  (which, as in Example~\ref{ex:filter}, we label ${\rm emit}_{x}$ for output
  each $x$).
  Figure~\ref{fig:pgraph-filter} shows a canonical example in which the
  property of interest is whether two agents, in an annulus-shaped environment
  with three beam sensors, are apart or not.
  \begin{figure}[t]
    \centering
    \includegraphics[scale=0.8]{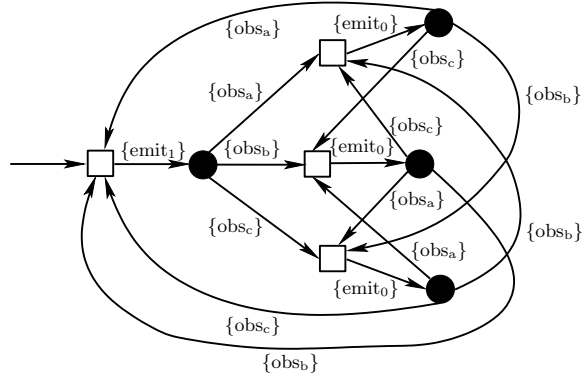}
    \caption{%
      The `agents together' filter devised by \citet{TovCoh+14} expressed as a
      p-graph.  The ${\rm emit}_0$ action indicates that the agents are
      separated by a beam, and ${\rm emit}_1$ indicates that the agents are
      together.
      \label{fig:pgraph-filter}
    }
  \end{figure}
\end{example}

\begin{example}[P-graphs for universal plans]
  The interaction languages for universal plans introduced in
  Example~\ref{ex:schoppers} can be cast as p-graphs in a straightforward way.
  The p-graph has a single observation vertex, with one uniquely-labeled
  out-edge corresponding to each world state, and one action state for each of
  the distinct available actions.  See Figure~\ref{fig:pgraph-universal}.
  \begin{figure}[t]
    \centering
    \includegraphics[scale=0.7]{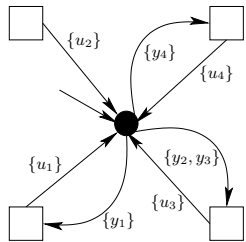}
    \caption{A universal plan expressed as a
    p-graph.\label{fig:pgraph-universal}}
  \end{figure}
\end{example}

\begin{example}[Erdmann-Mason-Goldberg-Taylor plans]
  Figure~\ref{fig:pgraph-emgt} shows an example of a sensorless manipulation
  plan, in the form described in Example~\ref{ex:emgt1}, expressed as a
  p-graph.
  Of particular note is the fact that this plan exhibits an unexpected
  dimension of nondeterminism: at each step it indicates sets of allowable
  actions, rather than a single predetermined one.  This degree of `choice' in
  the actions appears in the interaction language as a large collection of
  individual event sequences, but is expressed compactly within the p-graph.
  Also of note is that, generally, the graphs of knowledge states searched to
  produce such plans are themselves p-graphs.
  \begin{figure}[t]
    \centering
    \includegraphics[width=0.7\textwidth]{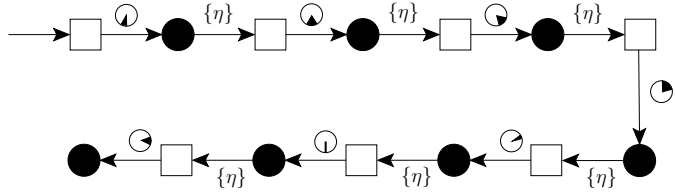}
    % Note: I used EMGT-raster for the original submission instead, because
    % IJRR's computer's were choking on Type 3 fonts in this figure.

    \caption{A plan for orienting an Allen wrench via tray tilting, expressed
    as a p-graph.  Action edges are labeled with sets of azimuth angles for
    the tray.  There is a single dummy observation, $\eta$.  This plan
    is shown as Figure~2 in~\citet{ErdMas88}.
    \label{fig:pgraph-emgt}}
  \end{figure}
\end{example}

\begin{example}[Nondeterministic graphs]\label{ex:erdmann}
  Recent work by Erdmann~\citeyearpar{erdmann10topology,erdmann12topology} encodes
  planning problems using finite sets of states, along with nondeterministic
  actions represented as collections of edges `tied' together into single
  actions.  One might convert such a graph to a p-graph by replacing each group
  of action edges with an observation node, with an outgoing observation edge
  for each edge constituting the original action.  Figure~\ref{fig:threefig}
  shows an example.
  
  \begin{figure}[t]
    \begin{minipage}[t]{0.3\textwidth}
      \centering       
      \includegraphics[scale=0.5]{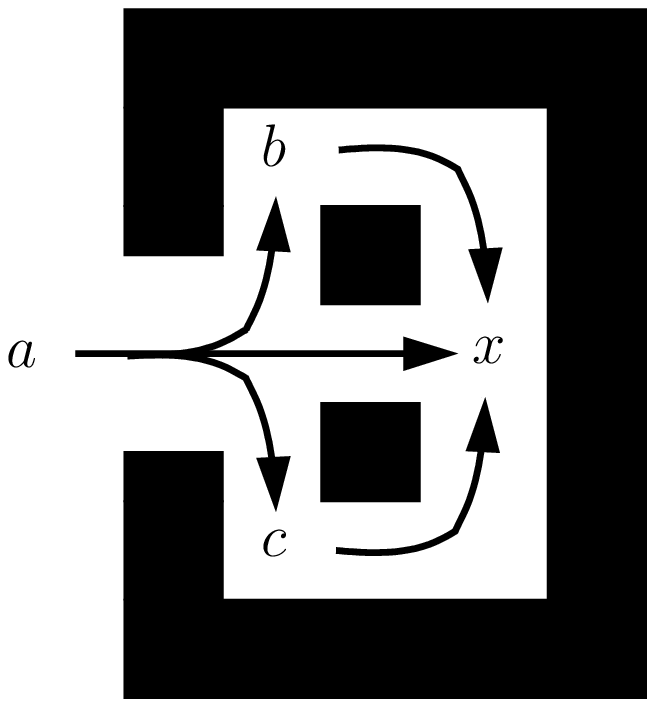}
    \end{minipage}  
    \begin{minipage}[t]{0.27\textwidth}
      \centering   
      \includegraphics[scale=0.68]{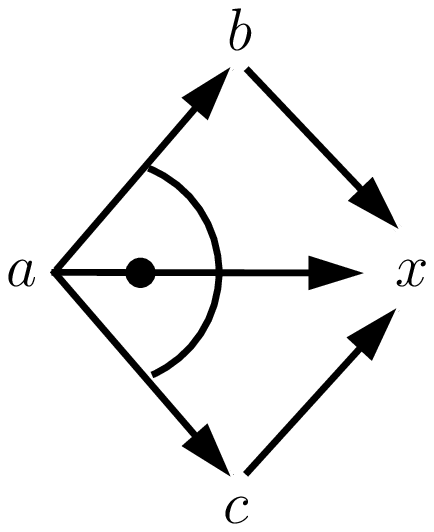} 
    \end{minipage}  
    \begin{minipage}[t]{0.3\textwidth}
      \centering   
      \includegraphics[scale=0.85]{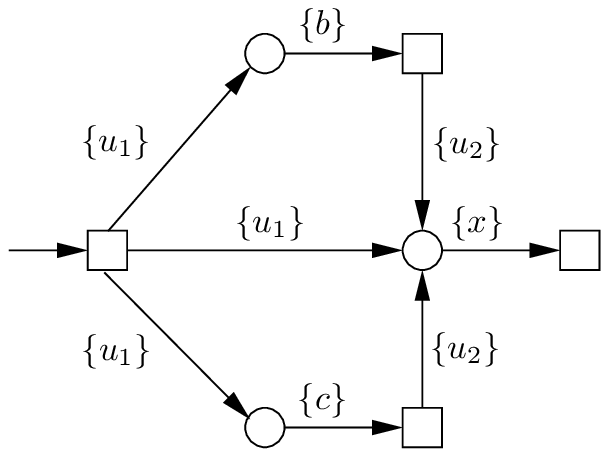} 
    \end{minipage}  
      \caption{[left] A planning problem due to Erdmann~\citeyearpar{erdmann12topology}. [middle] The nondeterministic graph of this problem. [right] An equivalent p-graph.
      \label{fig:threefig}}

  \end{figure} 
\end{example}

% For the journal version: include AND-OR graphs 
% For the journal version: include a hybrid automaton

The intent in these examples is to illustrate that p-graphs form a general
class that unifies, in a relatively natural way, a number of different kinds of
objects that have been studied over a long period of time.  The particular
constraints applied in each case impose certain kinds of structure that proved
useful in the original context.  

%\subsection{The interaction language induced by a p-graph}

While graph and graph-like objects appear in the prior work in various guises,
few have formalized the semantics of those objects by connecting them
to the languages they induce. The definitions we present next make precise the
way in which a p-graph is an implicit definition of an interaction language.

\begin{definition}[transitions to]\label{def:transto}
  For a given p-graph $G$ and two states $v,w \in V(G)$, a sequence of events $\eseq$ \emph{transitions in $G$ from $v$ to $w$} if there exists a sequence of
  states $v_1,\ldots,v_{k+1}$, such that
  \begin{enumerate}
    \item $v_1=v$,
    \item $v_{k+1} = w$, and
    \item for each $i=1,\ldots,k$, there exists an edge
    $\edge{v_k}{E_k}{v_{k+1}}$ for which $e_k \in E_k$.
  \end{enumerate}
\end{definition}

\noindent The states $v$ and $w$ need not be distinct; for every $v$, the empty
sequence transitions in $G$ from $v$ to $v$.  Longer cycles may result in
non-empty sequences of states that start at some $v$ and return.

\begin{definition}[valid]\label{def:valid}
  For a given p-graph $G$ and a state $v\in V(G)$, 
  a sequence of events $\eseq$ is \emph{valid} from $v$ if there
  exists some $w \in V(G)$ for which $\eseq$ transitions from $v$ to $w$.
\end{definition}

\noindent Observe that the empty sequence, $\emptyseq$, is valid from all states in any p-graph.

\begin{definition}[execution]
  An \emph{execution} on a p-graph $G$ is a sequence of events valid from some
  start state in $\Vinit(G)$.
\end{definition}

The preceding definitions prescribe when a sequence is valid on a p-graph,
placing few restrictions on the sets involved. There are several
instances of choices recognizable as forms of non-determinism: 
(i)\,there may be multiple elements in $\Vinit$; 
(ii)\,from any $v \in \Vu$ some action $u$ may be an element in sets on multiple outgoing action edges; 
(iii)\,similarly, from any $w \in \Vy$ some observation $y$ may qualify for multiple outgoing observation edges.

We can now `close the loop' between interaction languages and p-graphs.

\begin{definition}[induced language]
\label{def:inducedlang}
  Given a p-graph $G$, its \emph{induced language} is the set of all of its
  executions.  It is denoted $\langof{G}$.
\end{definition}

\begin{theorem}[induced languages are interaction languages]
  For any p-graph $G$, the induced language $\langof{G}$ is an interaction
  language.
\end{theorem}
\begin{proof}
  Because $G$ is bipartite, with its states partitioned into action states and
  observation states, all of its executions alternate between actions and
  observations.  Moreover, if $\set{\Vinit} \subset \set{\Vu}$, then its
  non-empty executions begin with actions, matching the first regular
  expression in Definition~\ref{def:ilang}.  If $\set{\Vinit} \subset
  \set{\Vy}$, then the non-empty executions if $G$ begin with actions, matching
  the second regular expression in Definition~\ref{def:ilang}.

  It remains only to confirm that $\langof{G}$ is prefix-closed.  Consider some
  execution $s = e_1 e_2 \cdots e_m \in \lang{L}$, and some prefix of $s$,
  denoted $s' = e_1 \cdots e_k$ with $k < m$.  We need to show that $s' \in
  \langof{G}$.
  We know from Definition~\ref{def:valid} that $s$ transitions from some start
  state $v$ to some final state $w$.  Therefore, via
  Definition~\ref{def:transto}, we know that there exists sequence of states
  $v_1,\ldots,v_{m+1}$ in $G$, with $v_1=v$, reached by followed edges labeled
  with events in $s$.  But considering only $v_1, \ldots, v_{k+1}$, we see that
  $s'$ is valid from $v_1$ as well.  Therefore, $s' \in \langof{G}$.
\end{proof}

% TODO: Not sure what's intended here.
% | This definition and the use of the word ``execution'' emphasizes the means by
% | which the p-graph's language may be produced constructively.

This theorem establishes a tight relationship between interaction languages and
p-graphs. Every p-graph induces an interaction language (though some
interaction languages cannot be expressed as p-graphs with finitely many states, cf.
Example~\ref{ex:notreg}).
Thus, we can meaningfully apply terms defined for interaction languages to
p-graphs as well:  Given two p-graphs $G_1$ and $G_2$, one might refer to the
set of joint executions (that is, joint event sequences) of $G_1$ and $G_2$.
We might say that $G_1$ and $G_2$ are akin to one another, or that $G_1$ is
finite on $G_2$, or that $G_1$ is safe on $G_2$.  These kinds of statements
should be read as referring to the interaction languages induced by the
p-graphs.

\begin{example}[Pentagonal world] 
  Figure~\ref{fig:pentaworld} presents concrete realizations of several of the
  preceding definitions in a single scenario.  A robot moves in a pentagonal
  environment. Information---at least at a certain level of
  abstraction---describing the structure of the environment, operation of the
  robot's sensors, its actuators, and their inter-relationships is represented
  in the p-graph associated with the scenario.
  The induced interaction language is
    $$ \langof{G} = \operatorname{Pref}((u_1y_1u_1y_1u_1y_1u_1y_2(u_2y_2)^\ks u_1y_1)^\ks), $$
  in which $\operatorname{Pref}(\cdot)$ denotes the prefix-closure of its language
  operand.
  Both filtering and planning questions can be posed as problems on this interaction
  language, as represented in this p-graph.
    \begin{figure}[t]
      \centering
      \includegraphics[scale=0.7]{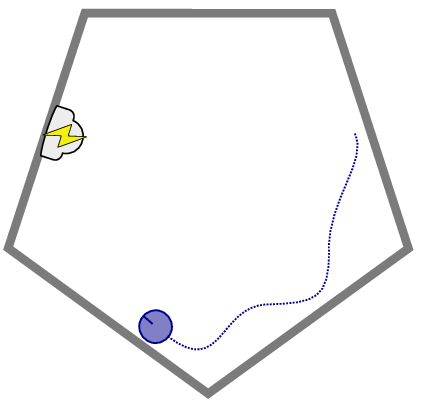}
      \includegraphics[width=0.38\textwidth]{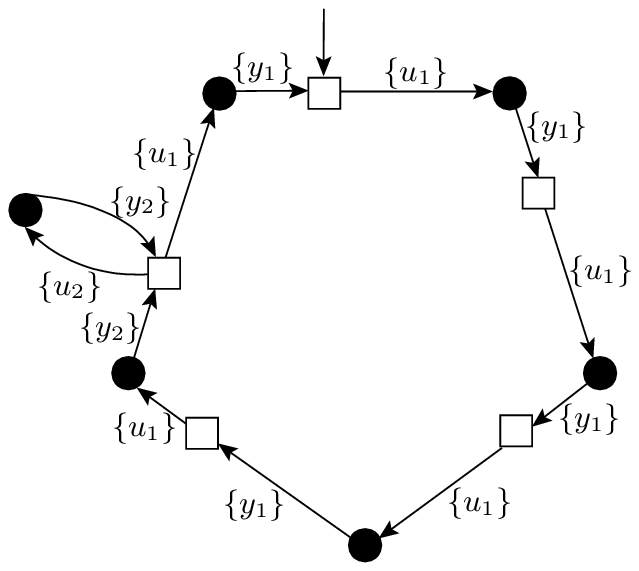}     
      \caption{[left] A robot wanders around a pentagonal environment; the segment with the lightning-bolt contains a battery charger.  [right] A p-graph model of this world.
      \label{fig:pentaworld}}
    \end{figure}
\end{example}

\subsection{Labels}

To keep the model amenable to direct algorithmic manipulation, we have required
that a p-graph consist of only finitely many states.  However, the labels for
each edge, either $\set{U}(e)$ or $\set{Y}(e)$, are sets that need not be finite.  This
detail is important for modeling real systems. For example, robots typically
have observation spaces which are large or infinite---including most nontrivial
real sensor systems---in which it would be, at best, computationally
intractable to list observations individually. The same can be said
for actions too.

We can permit labels that describe infinite sets if some simple operations on
the set algebra over $\set{U}$ and $\set{Y}$ are available.  
Any observation-originating edge $e$ is labeled with the set $\set{Y}(e)$,
such that $\set{Y}(e) \in \pow{\set{Y}}$. The analogous relation holds for 
action-originating edges too.
In what follows, we assume that both
$\pow{\set{U}}$ and $\pow{\set{Y}}$ 
are equipped with the following six operations:\footnotemark

\footnotetext{To save presenting distracting technicalities,
we use $\pow{\set{U}}$ for the set algebra over $\set{U}$, though the need for
finite constructions usually means that the algebra is a proper subset of the powerset.}

\begin{enumerate}
  \item[1--3.] \textsc{Union}, which accepts two labels and computes a new label 
  representing their union, along with \textsc{Intersection} and
  \textsc{Difference}, which operate {\it mutatis mutandis} for the
  set intersection and set difference operations.

  \item[4.] \textsc{Empty}, which accepts a label and returns \textsc{True} if
  and only if the label represents the empty set.

  \item[5.] \textsc{Contains}, which accepts a label and an event,
and decides whether that event is member of the set represented by
  that label.

  \item[6.] \textsc{Representative}, which accepts a non-empty label and returns an
  event contained in the set represented by that label.
\end{enumerate}
Any data structure capable of answering these queries is suitable for
representing the labels in the algorithms in this paper.  
Some examples follow.

\begin{example}\label{ex:interval}
  Suppose $\set{U} = \real$ or $\set{Y} = \real$.  Since each label
  should represent a set of real numbers, one option is to let each represent a
  finite union of real intervals.  The intervals may be bounded or unbounded.
  Each interval may also be open, closed, or half-closed.
  Figure~\ref{fig:intervals} shows an example.
  To represent a label from this label space, we use a data structure with three parts:
  \begin{enumerate}
    \item A list of $n \in \natzero$ real number \emph{endpoints} $e_1,\dots,e_n \in \real$.

    \item A list of $n+1$ boolean \emph{interval flags} $f_1,\dots,f_{n+1}$.
    The interpretation is that, for each $1<j<n$, the real numbers between $e_j$
    and $e_{j+1}$ are included in the set if and only if $f_j$ is \textsc{True}.
    At the extremes, real numbers less than $e_1$ are in the set when $f_1$ is
    \textsc{True}, and likewise numbers greater than $e_n$ are in the set when
    $f_n$ is \textsc{True}.

    \item A list of $n$ boolean \emph{endpoint flags} $p_1,\dots,p_n$, with the
    semantics that, for any $1 \le j \le n$, the real number $e_j$ is in the
    label's observation set if and only if $p_j$ is \textsc{True}.
  \end{enumerate}
  Note that any finite union of real intervals (including, for example, the
  empty set and the full real line, which have \mbox{$n=0$}) can be expressed in this
  format.

  \begin{figure}
    \begin{center}
      \includegraphics[width=\columnwidth]{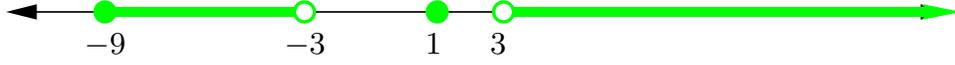}
    \end{center}
    \newcommand{\T}{\textsc{True}}
    \newcommand{\F}{\textsc{False}}
    \caption{An interval label for the set
      $[-9,-3) \cup \{ 1 \} \cup (3,\infty)$.
    The label data structure has 4 endpoints
      $(-9, -3, 1, 3)$,
    5 interval flags
      $(\F, \T, \F, \F, \T)$,
    and 4 endpoint flags
      $(\T, \F, \T, \F)$.
    }
    \label{fig:intervals}
  \end{figure}

  The \textsc{Union}, \textsc{Intersection}, and \textsc{Difference} operations
  can be implemented by performing a left-to-right sweep, adding endpoints and
  flags appropriately to the resulting label.  The \textsc{Empty} method requires a
  simple check for any endpoint flags or interval flags that are \textsc{True}.
  The \textsc{Contains} check can be implemented by a binary search for the
  correct interval, followed by a check against the relevant flag.
  \textsc{Representative} should return an element, either an endpoint or in the
  interior of an interval (in the general case, perhaps the midpoint between two
  endpoints) for which the corresponding flag is \textsc{True}.
\end{example}

\begin{example}
  Labels that represent a finite number of events---as is the case for many
  simple sensors such beam detectors or bump sensors, or simple actuators with
  a discrete modes of operation---can be modeled by storing the elements explicitly
  in almost any container data structure, such
  as a balanced binary tree or a hash table.
\end{example}

\begin{example}
\label{ex:productlabels}
  We expect that a common case will involve action or observation sets that are
  composed, via Cartesian product, from simpler sets.  That is, we may generally have
    $ \set{X} = \set{L}_1 \times \dots \times \set{L}_m, $
  in which each $\set{L}_i$ is a set for which we have the requisite operations,
  and $\set{X}$ is $\set{U}$ or $\set{Y}$.
  In such a case, we can define a set algebra over $\set{X}$ in which each
  label represents a union of Cartesian products of sub-labels, in the form
    $ \bigcup_i \left( \ell_1^{(i)} \times \dots \times \ell_m^{(i)} \right)$,
    $\ell_k^{(i)} \subseteq \set{L}_k$, where $i \in \{1,\dots,n\}$. 
  Under this representation, a \textsc{Union} between labels becomes a mere
  concatenation of Cartesian product lists.  The \textsc{Intersection}
  operation requires pairwise intersections between each of the constituent
  Cartesian products of each of the two labels:
    \begin{multline*}
      \left[ \bigcup_{\!\!\phantom{j}i} \left( \ell_1^{(i)} \times \dots \times \ell_m^{(i)} \right) \right]
        \medcap
      \left[ \bigcup_j \left( p_1^{(j)} \times \dots \times p_m^{(j)} \right) \right]
      \\ = \bigcup_i \bigcup_j \left( (\ell_1^{(i)} \cap p_1^{(j)}) \times \dots \times (\ell_n^{(i)} \cap p_m^{(j)}) \right).
    \end{multline*}
  The \textsc{Difference} operation is similar, but first requires a refinement
   of the labels (see Section~\ref{sec:refine} below) along each dimension.
\end{example}

Example~\ref{ex:productlabels} also illustrates that
while it is natural to think of labels as the descriptions of sets borne on
edges, such as either $\set{U}(e)$ or $\set{Y}(e)$ for some $e$, it is also
meaningful to think of sets which are basic constituents from which to make up
such labels. For this reason, in what follows we use the general notation of
$\ell_i$, which can describe either sets of actions or observations.

\subsubsection{Label refinement}\label{sec:refine}

Several of the algorithms in subsequent sections rely
on a subroutine to compute of a \emph{refinement} of a set of labels.
Specifically, in several places we need an algorithm that
 accepts as input an unordered set of labels $\ell_1,\dots,\ell_n$, and
 produces as output an unordered set of labels $\ell'_1, \dots,
  \ell'_m$, such that
    $ \bigcup_i \ell_i = \bigcup_j \ell'_j $
  and, for each $\ell' \in \{\ell'_1, \dots, \ell'_m \}$ and each $x_1, x_2
  \in \ell'$, we have
    \begin{equation*}
      \big\{ \ell \in \{ \ell_1, \dots, \ell_n \} \mid x_1 \in \ell \big\}
       = \big\{ \ell \in \{ \ell_1, \dots, \ell_n \} \mid x_2 \in \ell \big\}.
    \end{equation*}
The intuition is, given a set of labels, to compute a partition of the
events spanned by those labels.  This partition should be fine enough to
separate the input labels from one another, in the sense that the set of
corresponding input labels is constant across all events in each output
label.  Such a partition is valuable because it enables us to `drop down' from
the level of sets to the level of individual events, by selecting a
\textsc{Representative} from each of the output labels, without danger of
missing any structure inherent to the input label set.

\begin{algorithm}[t]
  \caption{\textsc{RefineLabels}($\ell_1,\dots,\ell_n$)}
  \label{alg:refine}
  \DontPrintSemicolon
  \SetKwInOut{Input}{Input}
  \SetKwInOut{Output}{Output}
  \Input{A list of labels, $\ell_1, \ell_2,\dots, \ell_n$.}
  \Output{A list of refined labels $\ell'_1, \ell'_2,\dots,\ell'_m$.}
  \BlankLine
  \tcp{Construct the union of all the input labels}
  {$r \leftarrow \ell_1$}\; 
  \For{$\ell \in \{ \ell_2,\dots,\ell_n\}$}{
    {$r \leftarrow \textsc{Union}(r, \{ \ell \} )$}\; 
  }
  $R \leftarrow \{ r \}$ \;
  \tcp{Refine each at at time}
   \For{$\ell \in \{ \ell_1,\dots,\ell_n\}$}{
    $R' \leftarrow \emptyset$\;
    \For{$r \in R$}{
      $R'.\textrm{append}($\textsc{Intersection}$(r, \ell))$;~\tcp{The part inside $r$\dots}
      $R'.\textrm{append}($\textsc{Difference}$(r, \ell))$;~\tcp{\dots and then the rest}
    }
    $R \leftarrow R'$\;
  }
  \Return{$R$}
\end{algorithm}

Algorithm~\ref{alg:refine} shows how one can perform this operation in a
general way, for any labels that support the \textsc{Union},
\textsc{Intersection}, and \textsc{Difference} operations.  The algorithm
starts with a single label representing the complete set of relevant
events, and then refines that partition using each of the input labels.

\begin{figure}[h]
\vspace*{1cm}
    \centering
  \includegraphics[scale=0.8]{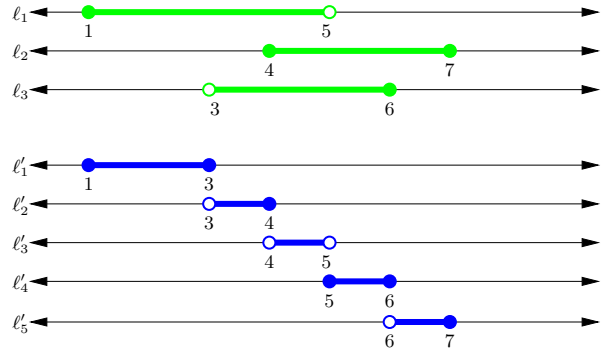}
  \caption{An example of label refinement with interval labels.  The three
  input labels, representing the overlapping intervals $[1,5)$, $[4,7]$, and
  $(3,6]$, are refined into six disjoint output labels
  $[1,3]$, $(3,4]$, $(4,5)$, $[5,6]$, and $(6,7)$.}
  \label{fig:interval_refine}
\end{figure}

We note, however, that for certain kinds of labels, such as the interval labels
introduced in Example~\ref{ex:interval}, it may be practical and efficient to
implement this operation directly, utilizing the internal details of the label
data structure, rather than this generalized approach.  For the interval label
space introduced in Example~\ref{ex:interval}, the label refinement operation
can be implemented directly.  Figure~\ref{fig:interval_refine} shows an example
of the computation.  The approach is to form a combined, sorted list of all of
the endpoints for each of the input labels, and then form the refined output
labels using a left-to-right sweep, starting a new label each time the set
input labels touched by the sweep line changes.

\subsection{Basic operations on p-graphs}

Next, we examine operations on p-graphs. Of particular
interest is the question of how these operations affect the induced interaction
language: some will mutate the language; others will preserve it.

We give an example of a constructive operation which produces a new p-graph with a
new interaction language, exploiting initial state nondeterminism.

\begin{definition}[union of p-graphs]\label{defn:pairproduct}
  The \emph{union} of two p-graphs $U$ and $W$, each akin to
  the other, denoted by $U \uplus W$, is the p-graph constructed by
  including both sets of vertices, both sets of edges, and 
    with initial states equal to $V_0(U) \cup V_0(W)$.
  %\begin{enumerate}
  %  \item vertex set equal to $V(U) \cup V(W)$.
  %  \item edges set equal to $\{ E(e) \mid e\in V(U)\} \cup \{ E(e) \mid e\in V(W)\}$
  %  \item the initial states equal to $V_0(U) \cup V_0(W)$.
  %\end{enumerate}
  \label{def:pairproduct}
\end{definition}

The intuition is to form a graph that allows, via the nondeterministic
selection of the start state, executions that belong to either $U$ or $W$.

\begin{theorem}
For p-graphs $P$ and $Q$: $\langof{P} \cup \langof{Q} = \langof{P\uplus Q}$
\end{theorem}
\begin{proof}
  Follows directly from Definitions~\ref{def:inducedlang} and \ref{def:pairproduct}.
\end{proof}

In general, the sets labeling two edges departing a vertex of a p-graph need
not be disjoint, allowing multiple `next' states to be indicated for the same
event.  It can be useful to distinguish p-graphs where this circumstance arises
from those where it is absent.  Our next definition formalizes this, while also
highlighting that multiple p-graphs can induce the same interaction language. 

\begin{definition}[state-determined]
  \label{defn:sd}
  A p-graph $P$ is in a {\em state-determined presentation} if $\card{\Vinit(P)}=1$ and 
    from every action vertex $u\in \Vu$, the edges $e_u^1, e_u^2, \dots, e_u^n$ originating at~$u$
          bear disjoint labels: $U(e_u^i)\cap U(e_u^j)=\emptyset, i\neq j$, and
    from every observation vertex $y\in \Vy$, the edges $e_y^1, e_y^2, \dots, e_y^m$ originating at~$y$
          bear disjoint labels: $Y(e_y^i)\cap Y(e_y^j)=\emptyset, i\neq j$.
\end{definition}

The intuition is that in a p-graph in a state-determined presentation it is
easy to determine whether an event sequence is an execution: one starts at the
unique initial state and always has an unambiguous edge to follow.  We note,
however, that the p-graph with a state-determined presentation for some set of
executions need not be unique.

\begin{algorithm}
  \caption{\textsc{ToStateDeterminedPresentation}(${G}$)}
  \label{alg:expansion}
  \DontPrintSemicolon
  \SetAlgoLined \SetKwInOut{Input}{Input} \SetKwInOut{Output}{Output}
  \Input{A p-graph $G$ with vertex set $\set{V}$ and starting set $\set{V_0}$.}
  \Output{An equivalent state-determined p-graph $G'$ with $\set{W}$ and $\set{W_0}$, respectively.}
  \BlankLine
  {Initialize  $\set{W}$, $\set{W_0}$, as empty}\;
  {$\set{\textrm{Corresp}}[\cdot] = \emptyset$} ~\tcp{Construct an empty map to associate vertices between p-graphs}
  {Add $v_0'$ to $\set{W}$ and $\set{W_0}$}~~\tcc{Construct an initial the vertex in $G$, preserving action-\newline \phantom{}\quad\quad\qquad\qquad~originating or observation-originating type of elements in~$\set{V_0}$} 
  {$\set{\textrm{Corresp}}[v_0'] \leftarrow \set{V_0}$} ~\tcp  {Associate $v_0'$ with all $v_0\in\set{V_0}$ in $G$}
  {Initialize queue ${Q} \leftarrow \set{W_0}$}\;
  \While{${Q}$ {\upshape not empty}}{
    {$s' \leftarrow Q.\textrm{pop}$}\;

  \BlankLine
  \tcp{Refine each label and determine which states each refinement maps to:}
        {$\set{L} \leftarrow$ All outgoing edge labels of $\set{\textrm{Corresp}}[s']$}\;
        {$\set{L'} \leftarrow \textsc{RefineLabels}(\set{L})$ ~\tcp{cf. Algorithm~\ref{alg:refine}}}
        {$\set{\textrm{Lab}}[\cdot] = \emptyset$} ~\tcp{Construct an empty map to associate refined labels to states}
        \For{$\ell' \in \set{L'}$}{
            {Determine the set of states reached by tracing event $\textsc{Representative}(\ell')$ from each
                $\set{\textrm{Corresp}}[s']$, adding them to $\set{\textrm{Lab}}[~\ell'~]$}\;
        }
  \BlankLine
  \tcp{Produce new states as needed:}
        \For{\upshape $s \in \set{\textrm{Lab}}[\ell]$ {\upshape for some} $\ell$}{
            \If{$t \in \set{W}$, {\upshape where} $t$ {\upshape corresponds with} $s$}{
                {Add $\edge{s'}{\ell}{t}$} ~~\tcp{Add transition on $\ell$ to $G'$}
            }
            \Else{
                {Create new state $t$ corresponding to $s$}~~\tcp{Type should correspond too}
                {$Q.\textrm{push}(t)$} ~~\,~\tcp{Add to queue to be processed}
                {Add $\edge{s'}{\ell}{t}$} ~\tcp{Add transition on $\ell$ to $G'$}
            }
        }
  }
  {\textbf{return} $G'$}\;

\end{algorithm}

%\begin{figure}
%  \centering
%  {~~\includegraphics[scale=0.43]{figures/Before_state_determined_expansion-tweaked}\hfill\phantom{}}\\
%  \vspace*{-59pt}
%  {\phantom{}\hfill\includegraphics[scale=0.43]{figures/After_state_determined_expansion-tweaked}~\quad}
%  \caption{[left] A p-graph that is not in state-determined normal form.
%  [right] Refinement and adjustment of labels ~\ref{alg:expansion} to this p-graph.}
%  \label{fig:sd}
%\end{figure}

\begin{figure}[t]
  \centering
  {~~\includegraphics[scale=0.43]{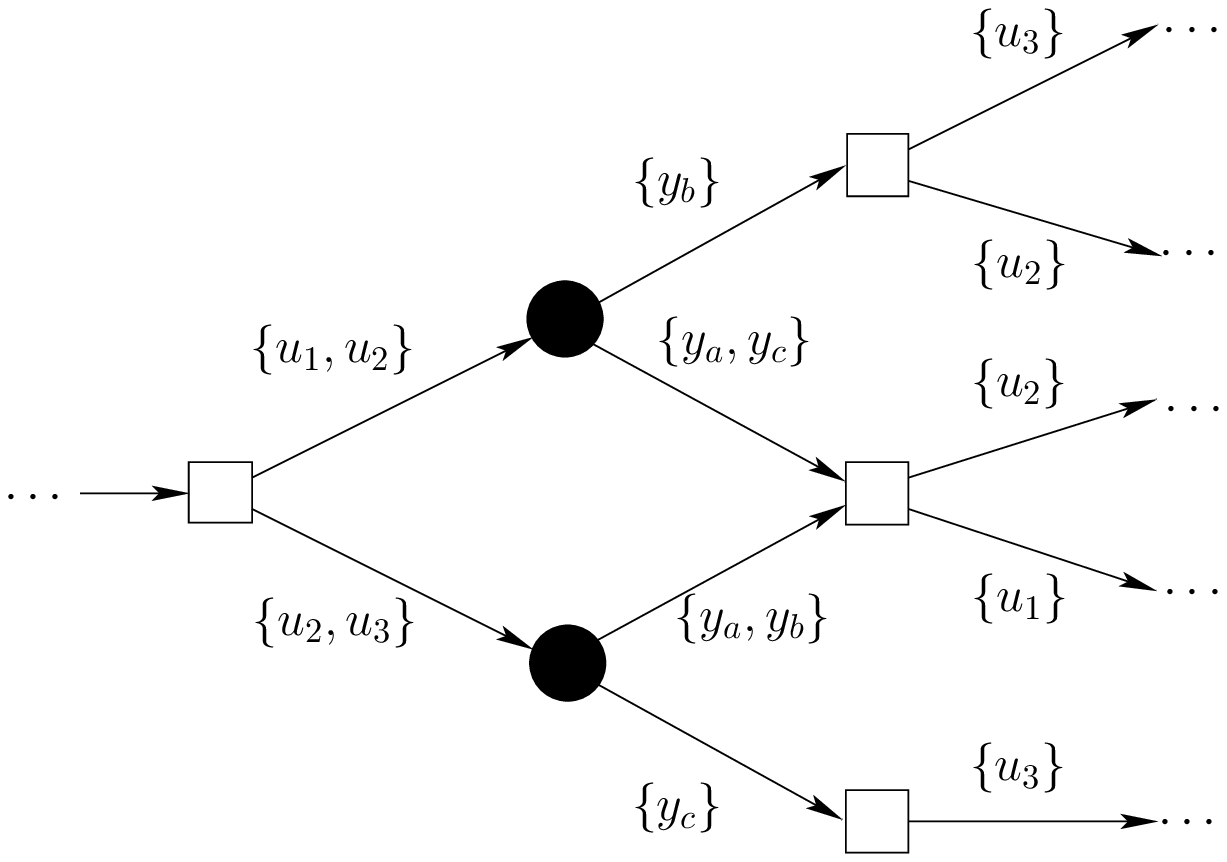}\hfill\phantom{}}\\
  \vspace*{-139pt}
  {\phantom{}\hfill\includegraphics[scale=0.43]{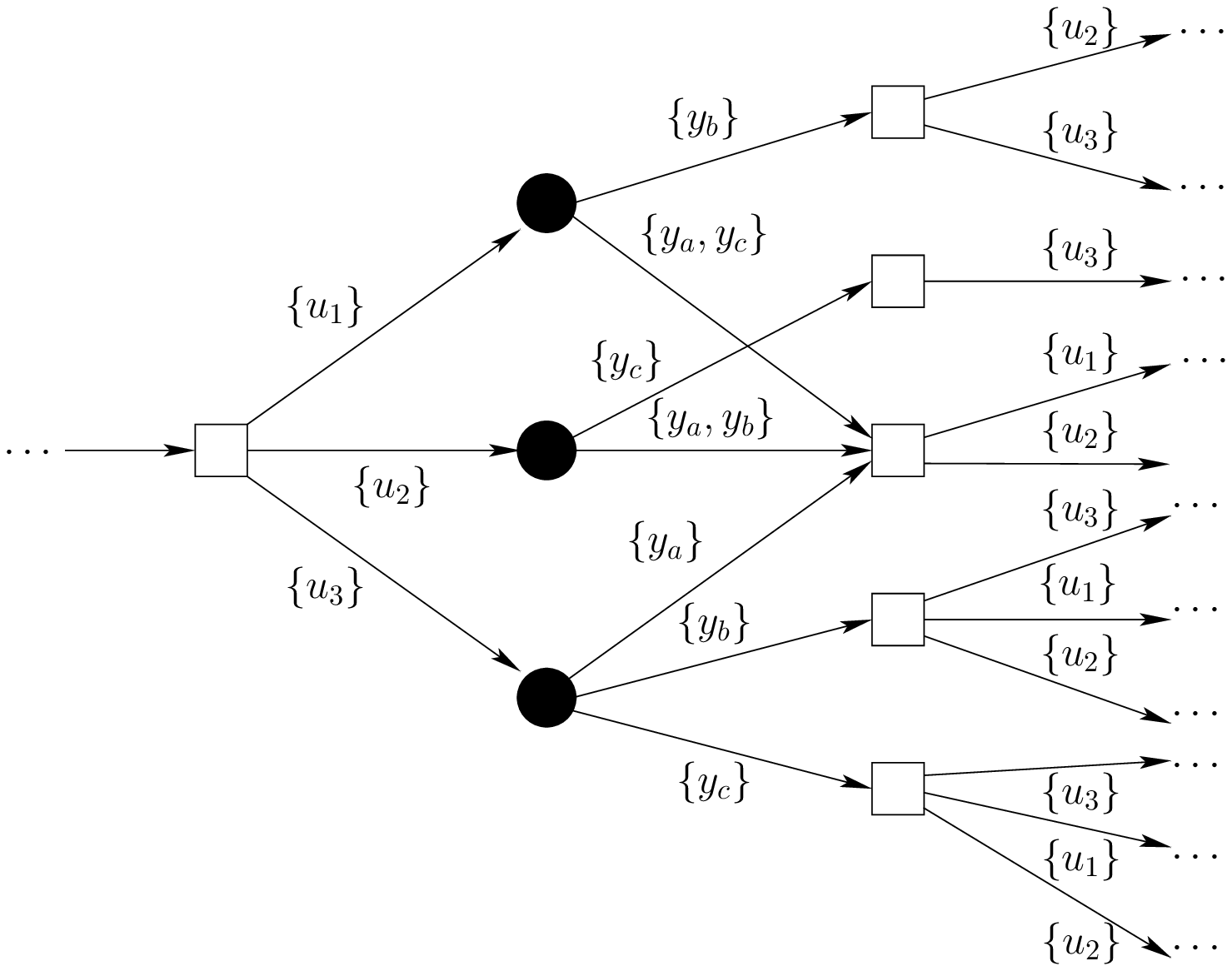}~\quad}
\caption{A fragment of a p-graph on the left is processed into the state-determined presentation on the right with the concomitant increase in number of states.
\label{fig:st}}
\end{figure}

Given any p-graph it is possible to construct a new p-graph that has the same
set of executions on it, but which is in a state-determined presentation.
Algorithm~\ref{alg:expansion} shows how to convert an arbitrary p-graph into
state-determined presentation.  The basic idea is a forward search that
performs a powerset construction on the input p-graph. We begin by constructing
a single state to represent the ``superposition'' of all initial states, and
push that onto a empty queue. While the queue has elements, remove a vertex and
examine the edges leaving the set of vertices associated with it in the
original input p-graph.  The labels on those edges are refined by
constructing a partition of the set spanned by the union of the labels in a way
that the subsequent sets of states in the input p-graph is clear. Edges are
formed with the refined sets connecting to their target vertices, constructing
new ones as necessary, and placing these in the queue.  This requires the use
of Algorithm~\ref{alg:refine} to ensure that the edges in the new filter are
drawn correctly.  Figure~\ref{fig:st} gives a simple example of the process for
part of a p-graph.  Though this shows a moderate increase in size, in general,
following the procedure above may produce a p-graph as output that has an
exponentially larger set of states than the input. 

% We've emphasized ways in which the state determined form is not unique 
% here's another way:
%   * we don't prohibit U(e) or Y(e) from being empty. Having it empty is
%   identical with omission of the edge in terms of the language. That means
%   unless the p-graph is fully connected, we get some choices: either adding
%   edges with \emptyset, or deleting edges entirely. 
% We could avoid this by insisting that the edges have non-emptysets. But I 
% like that you can apply a label map which results in an empty set, without
% requiring a special cases for removal of empty set edged.
%

\section{Label maps}\label{sec:maps}

We express modification of capabilities through maps that mutate the labels
attached to the edges of a p-graph.

\begin{definition}[action, observation, and label maps]\label{def:set-map} An
\emph{action map} is a function $h_u:\set{U} \to \pow{\set{U'}} \setminus
\{\emptyset\}$ mapping from an action space $\set{U}$ to a non-empty set of actions
in a different action space $\set{U}'$.  Likewise, an \emph{observation map} is a
function $h_y:\set{Y} \to \pow{\set{Y}'} \setminus \{\emptyset\}$ mapping from
an observation space $\set{Y}$ to a non-empty set of observations in a different
observation space $\set{Y}'$.  A \emph{label map} combines an action map $h_u$ and a
sensor map $h_y$:
\begin{equation*}
    h(a) = \begin{cases}
      h_u(a) & \text{if } a \in \set{U} \\
      h_y(a) & \text{if } a \in \set{Y}
    \end{cases}.
  \end{equation*}
\end{definition}

% Note: we had to include the non-empty bit, so that edges can't disappear. When they
%       were able to, maps could introduce homomorphism

It is useful to extend this notion, so we do this immediately.

\begin{definition}[label maps on sets and p-graphs]
  Given a label map $h$, its \emph{extension to sets} is
  a function that applies the map to a set of labels:
    $$h(E) = \bigcup_{e \in \set{E}} h(e).$$
  The \emph{extension to p-graphs} is a function that mutates p-graphs by
  replacing each edge label $\set{E}$ with $h(\set{E})$. We will write $h(P)$
  for application of $h$ to p-graph $P$.
\end{definition}

\begin{example}[label maps on intervals]
  Representation of the action or observation spaces that are
  $\real$ via unions of intervals, 
  as detailed in Example~\ref{ex:interval}, 
  lends itself to definition of label maps.
  To represent a label map on such an event space, we might, for example,
  take bounding polynomials $p_1(x)$ and $p_2(x)$, and define
    $$ h(x) = \{ x' \mid p_1(x) \le x' \le p_2(x) \}. $$
  Given a finite-union-of-intervals label $\ell \subset \mathbb{R}$, we can
  evaluate this kind of $h$ by decomposing $h$ into monotone sections,
  selecting the minimal and maximal values of $p_1$ and $p_2$ within that
  range, and computing the union of the results across all of the monotone
  sections.
  Figure~\ref{fig:map_interval} shows an example.
  \begin{figure}[t]
    \centering
      \includegraphics[scale=1.0]{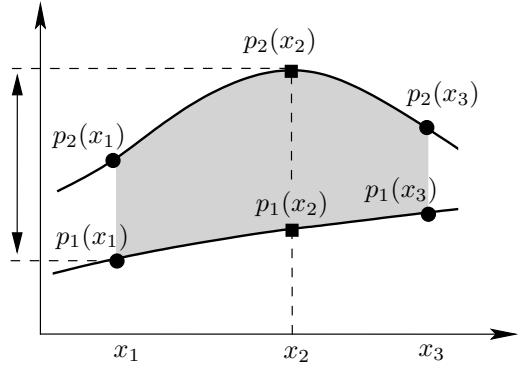}  
      \caption{A label map from $\mathbb{R}$ to $2^\mathbb{R}$ may be described by functions $p_1$ and $p_2$ as lower and
      upper bounds, respectively.  The marked vertical interval, spanning
      $p_1(x_1)$ to $p_2(x_2)$ illustrates the image of $h$ across the monotone
      segment from $x_1$ to $x_2$.  Values for other monotone segments would be
      computed similarly.
      \label{fig:map_interval}}
  \end{figure}
\end{example}

%\subsection{Maps that inflict damage}

Label maps allow one to express weakening of capabilities as follows.  If
multiple elements in the domain of $h(\cdot)$ map to sets that are not
disjoint, this expresses a conflation of two elements that formerly were
distinct.
\begin{enumerate}
\item When they are observations, this directly models a sensor undergoing a
  reduction in fidelity since the sensor loses the ability to distinguish
  elements.
  \item When they are actions, this models circumstances where uncertainty
  increases because a single action can now potentially produce multiple
  outcomes, and the precise outcome is unknown until after its execution.
\end{enumerate}

Further, when the image of element $\set{E}$ is a set with multiple constituents,
this also expresses the fact that planning becomes more challenging.
\begin{enumerate}
\item[3.] For observations, it means that several observations may result from
  the same state and, as observations are non-deterministic, this increases the
  onus for joint-executions to maintain safety (for example, plans must account
  for more choices).

\item[4.] For actions, while there is a seemingly larger choice of actions, this
  increase does not represent an increase in control authority because several
  actions behave identically.  
\end{enumerate}
  
  In both action and observation instances, the
  map may become detrimental when the outputs of $h(\set{E})$ intersect for multiple
  $\set{E}$s and thus `bleed' into each other.  Broadly, one would expect that this
  is more likely when the output sets from $h(\cdot)$ are larger.

The next two sections address questions of how to reason about this sort of
destructiveness for filtering and planning problems, respectively.

\section{Destructiveness in filters}\label{sec:filters}
The management of uncertainty via integration of sensor readings has been
a central theme in robotics research for decades. It is, thus, worth
examining how p-graphs might be specialized to express structures suitable
for such operations.  Earlier,
Example~\ref{ex:filter} (along with Example~\ref{ex:combfilter} presented
thereafter) showed how both interaction languages, generally, and p-graphs, in
particular, can describe estimation processes in the form of filters.  The word
filter is most familiar as a term used to describe practical
estimation components of robots and their controller software. The filters
treated in this section subsume those, representing a larger class, providing a
broad, abstract
theoretical treatment of algorithmic processes that aggregate information.  

A p-graph over event space $\set{E}$ is useful as a filter if the elements of
$\set{U}$, the actions, are interpreted as merely publishing or emitting
information.  The idea is that filters influence an agent's representation of
state, rather than altering the underlying state of the world itself.  This act
of interpretation gives p-graphs for filtering special significance to the
agent that beholds them. (The idea that some \emph{interpretation} must be
provided to apply a p-graph to some circumstance is an important recurring
theme in this work.) In this section, we will use the word output, rather than
action, to emphasize the interpretation of a p-graph as a filter; occasionally
we also just use the word filter to refer to such a p-graph.

When a p-graph that is being used as a filter has some edge $e$ bearing a
non-singleton set $\set{U(e)}$, or when such a graph has an action edge with
multiple out-edges, the resulting language includes choices for the information
to be emitted.  That choice is made arbitrarily, so whoever the consumer of the
outputs might be (perhaps it is a controller or a planner), it must be able to
operate with any in the set.  Of course, if the p-graph is to be a faithful
estimator this will constrain the p-graph's  $\set{U(e)}$ sets. One expects
that a p-graph acting as a filter would produce information that is sound given
the stream of inputs seen; such a filter can produce multiple outputs so long
as the information from the filter need not be `tight.'  In the argot used to
describe probabilistic filters, this corresponds, roughly, to fact that many
possible filters may satisfy an unbiasedness criterion, though relatively few
that satisfy only that constraint are actually good estimators.

The previous discussion notwithstanding, it is far more usual to imagine a
unique output being produced in response to a particular history of events. We
find it useful to be precise about the p-graphs which, for any action vertex,
do not have any non-determinism on the output produced:

\begin{definition}\label{def:single-output}
% First take:
%	A p-graph $G$ is {\em single-outputting} if, for every edge $e$ originating from all $v \in \set{\Vu}$ reached by an execution, $\card{\set{U(e)}}=1$.
% Second take:
A p-graph $G$ is {\em single-outputting} if, for all $v \in \set{\Vu}$ reached by an execution, there is at most one edge $e$ originating at $v$, and it bears a set $\set{U(e)}$ with $\card{\set{U(e)}}=1$.
% Third take:
%A p-graph $G$ is {\em single-outputting} if the execution of every observation-terminal sequence in $\langof{F}$ arrives at a $v \in \set{\Vu}$, where there is at most one edge $e$ originating at $v$, and it bears a set $\set{U(e)}$ with $\card{\set{U(e)}}=1$.
\end{definition}

Though much previous discussion has emphasized the ability to use labels that
describe infinite sets,  the following establishes that this is not needed for
dealing with the outputs of filters if they are single-outputting.  There is
still, however, significant value in use of infinite sets of observations in
these cases.

\begin{theorem}[finiteness of single outputting p-graphs]
  For any single-outputting p-graph $G$, there is a single-outputting
  p-graph $G'$ that is equivalent in the sense that $\langof{G} = \langof{G'}$, but where 
  $G'$ is defined over an event space with finite $\set{U}$.
\end{theorem}
\begin{proof}
When $G$ does not have a finite $\set{U}$, 
one constructs $G'$ by copying $G$ and simply restricting the set of actions
for $G'$ to be the union of the $\set{U(e)}$ for edges originating at action
vertices in the executions. This $\set{U}$ is finite for there are finitely many edges
and each $\set{U(e)}$ contributes no more than one element.
\end{proof}
% We didn't actually mentioned that the edge set in p-graphs were finite, so
% I've added "finite" to the p-graph definition. I don't think whether someone
% interprets the edge set as finite or not should hinge on the (non-standard)
% understanding of whether it is a multi-graph, in the sense of allowing more than
% one edge between v_i and v_j.

% Shrink the stuff over-set on the edge
\newcommand{\mysetedge}[3]{\edge{#1}{\scalebox{0.5}{${#2}$}}{#3}}

\begin{algorithm}
  \caption{\textsc{ToSingleOutputtingPresentation}(${F}$)}
  \label{alg:singleton}
  \DontPrintSemicolon
  \SetAlgoLined
  \SetKwInOut{Input}{Input}
  \SetKwInOut{Output}{Output}
  \Input{A p-graph $F$ over an event space~$\set{E}$ with finite $\set{U}$.}
  \Output{An equivalent filter $F'$ that is single-outputting.}
  \BlankLine
  {Copy $\set{\Vy}$ to $F'$}\; 
  \For{\text{\bf every pair} $v_o,v_{o'}\in\set{\Vy}$, \text{\rm where} $\mysetedge{v_o}{Y(v_0,v_a)}{v_{a}}$$\mysetedge{\!}{U(v_a,v_{o'})}{v_{o'}}$}{
    \For{$i \in \set{U}(v_a,v_{o'})$}{
              Add action vertex $v_i$ to $\set{\Vu}$ for $F'$\;  
              Add edge $\mysetedge{v_o}{Y(v_0,v_a)}{v_i}$ to $F'$\;
              Add edge $\mysetedge{v_i}{\{i\}}{v_{o'}}$ to $F'$\;
    }
  }
  \If{$F$ \text{\rm is an action-first p-graph}}{
        \For{\text{\bf every } $v_1\in\set{\Vinit}$, \text{\rm where} $\mysetedge{v_1}{U(v_1,v_{o})}{v_{o}}$}{

            \For{$i \in \set{U}(v_1,v_{o})$}{
              Add action vertex $v_i$ to $\set{\Vinit}$ and $\set{\Vu}$ for $F'$\;  
              Add edge $\mysetedge{v_i}{\{i\}}{v_{o}}$ to $F'$\;
            }
        }
  }
  \Else {
        Copy $\set{\Vinit}$ from $F$ to $F'$.
  }

  {\textbf{return} $F'$}\;
\end{algorithm}

Moreover, there is a sort of converse that is true too. If the set $\set{U}$ is
finite, then having at most one singleton output set at each vertex, while a seemingly
significant constraint, does not limit the expressivity of such filters.  Every
finite action-space p-graph has an equivalent single-outputting presentation.

One can convert an arbitrary finite action-space p-graph to an equivalent
single-outputting p-graph by making duplicates of each action vertex, one for
each output in the original, and making the sole transitions from those new vertices
carry a single output. Algorithm~\ref{alg:singleton} gives the procedure in
detail.

There is a pattern worth noting here.  Definition~\ref{defn:sd} details a
particular structure that certain p-graphs possess, and
Algorithm~\ref{alg:expansion} then shows how any p-graph can transformed into a
p-graph with that structure. For finite output sets,
Definition~\ref{def:single-output} gives a particular structure that certain
p-graphs possess, and Algorithm~\ref{alg:singleton} shows how one can be
transformed into a form with that structure. Both operations transform p-graphs
whilst preserving the interaction languages they induce, which is why we call
them presentations.  There is more: Definitions~\ref{defn:sd}
and~\ref{def:single-output} describe two distinct ways of presenting a p-graph,
each of which places some restrictions on the kind of nondeterminism directly
expressed in the p-graph, but they are, in a certain sense, duals of one
another.  Algorithm~\ref{alg:singleton} may, in splitting edges with
non-singleton labeled action sets, introduce some observation labels which
overlap or multiple initial states.
Algorithm~\ref{alg:expansion} may, in eliminating overlapping labels on edges
incident the same observation vertex (and also, in eliminating multiple initial
states), produce a result which has action-states that have multiple edges
departing it. 

We claim that a certain class of p-graphs, however, can be represented in a way
that is simultaneously single-outputting and state-determined.\footnotemark
\footnotetext{Proof of the claim that the class of p-graph defined in Definition~\ref{def:deterministic} actually is this class
appears in Theorem~\ref{thm:so_is_enough_for_det_if_sd}
below; first we elaborate on the version of determinism we define and, in
particular, its relation to the classical notion.} We call these filters
\emph{deterministic}.

\begin{figure}[t]
  \centering
  \includegraphics[scale=0.48]{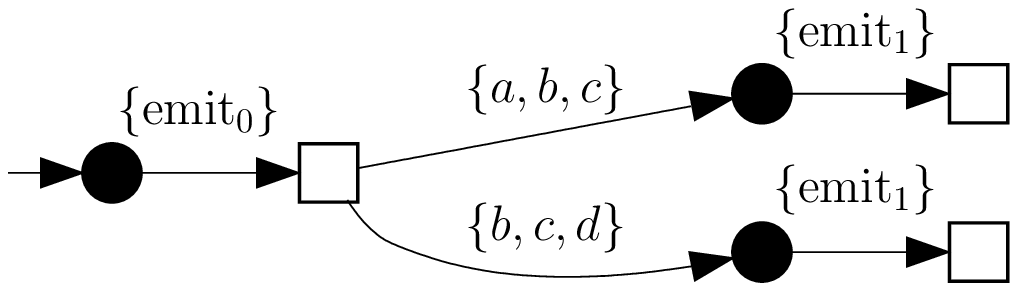}
  \quad
  \includegraphics[scale=0.48]{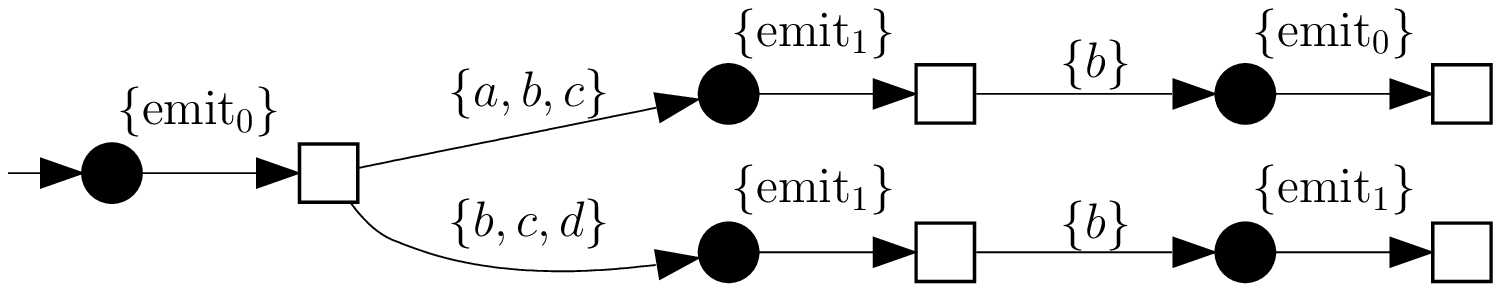}
  \caption{[left] A p-graph representation of a filter that, despite some
  labels not being disjoint, is a deterministic filter.  [right] A
  p-graph that is not deterministic.  To see this, note that the event sequence 
    `${\rm emit}_{0}$ $b$ ${\rm emit}_{0}$ $b$' has 
  two distinct successors: `${\rm emit}_{0}$' and `${\rm emit}_{1}$'.}
  \label{fig:deterministic}
\end{figure}

\begin{definition}\label{def:deterministic}
  \label{def:det-filt}
  A p-graph $F$ is a {\em deterministic filter} if every
  observation-terminal sequence has at most one
  successor in $\langof{F}$.
%
% This is the literal translation; of course since determinism is a property of the
% interaction language, it doesn't need to be this complex
%
%  if, for every two equilengthed sequences $e_1 e_2 e_3\cdots e_m$ 
%  and $f_1 f_2 f_3\cdots f_m$, both in $\langof{F}$, 
%
%  equality on the observation events ensures equality on outputs too:
%  \[ \mathop{\bigwedge}\limits_{i} \left(e_i \in \set{Y}\implies e_i = f_i \right) 
%    \implies  
%    \mathop{\bigwedge}\limits_{i} \left( e_i = f_i \right).
%  \]
\end{definition}

In the sense we have defined here, the property of being a deterministic filter
is a property of the p-graph's function, rather than a property of its
representation; it is thus fitting that this notion can be expressed in
terms of  the induced interaction language.  Note that the determinism does not
mean that each observation-terminal sequence arrives at a unique state, but
only that each observation sequence, if it yields any output, yields a single,
determined output. (Figure~\ref{fig:deterministic} helps clarify this
distinction by illustrating the difference.) This second idea is closely tied
to the usual sense of the word deterministic in classic automata theory:
the notion that sequences of observations (note, solely, observations) drive
the transitions of the automata, including dictating the precise state reached.
This form is also important because the filters we described in
Examples~\ref{ex:filter} and~\ref{ex:combfilter} are typically deterministic in
this more traditional sense---no arbitrary choices need be made during their
execution, observation inputs command the behavior.
We define this class as follows:

\begin{definition}\label{def:practicable}
  \label{def:stand-filt}
  A single-outputting p-graph $F$ is a {\em practicable filter} if 
  the validity of every sequence $\eseq \in\langof{F}$ is
  the consequence of precisely one sequence of state transitions in $F$.
%  or\\
%  every execution $\eseq \in\langof{F}$ is valid from 
%  only one $v_0 \in \Vinit(G)$ and only along one path in $F$.
%  or\\
%  every execution $\eseq \in\langof{F}$ has exactly one sequence of state
%  transitions from a start state in $v_1 \in \Vinit(G)$ to 
%  some state $v_{k+1}$, so $e_i \in E_i$ for each $\edge{v_i}{E_i}{v_{i+1}}$.
\end{definition}

They are named practicable because such filters are directly amenable to
implementation. This, no doubt, goes some way to explaining why the filters
that have appeared in the robotics literature are of this form.  

Naturally, there is a connection between these practicable filters and the
preceding notion of determinism, captured by these two lemmas:

\begin{lemma}
 Every practicable filter is a deterministic filter.
  \label{lem:pract-are-det}
\end{lemma}
\begin{proof}
Every observation-terminal sequence in the induced language traces a single
trajectory through the filter's states and so arrives at precisely one state. It is single-outputting 
so there is at most one edge from that vertex and, hence, at most one successor.
\end{proof}

\begin{lemma}
  For every deterministic filter $F$
  there exists a practicable filter $F'$ with $\langof{F}=\langof{F'}$.
  \label{lem:det-filts-are-pract}
\end{lemma}
\begin{proof}
  One constructs $F'$ by taking sets of vertices in $F$ as 
  vertices for $F'$; the start state of $F'$ is $\set{\Vinit}$.
  % That is \Vinit_{F'} is {\Vinit_{F}} 
  One adds edges in $F'$ by exploring each of the (finite) sequence
  prefixes that arrive at vertices in $F$, and labeling the transitions that
  are made. Tracing a prefix string on $F$ may result in a branch: two edges
  leaving an observation vertex may have labels which overlap (though this
  cannot happen with action vertices). At such choice points
  both choices should be taken, which is why the states in $F'$ 
  are subsets of vertices of $F$. 
  An edge in $F'$ originating from action vertex labeled, say, $\{v_i, v_j\}$
  (being associated with both action vertex $v_i$ and action vertex $v_j$ in
  $F$), only ever bears one label because the edge originating from 
$v_i$  and from $v_j$  in $F$ must produce the same output, otherwise  
otherwise $F$ would not be a
deterministic filter.
\end{proof}

Having established that deterministic filters may be of practical importance
because they can, ultimately, be turned into practicable filters, next we 
establish a relationship between the set of deterministic filters and the
presentations (state-determined and single-outputting) introduced earlier.

% Under the p-graph model the following are theorems translated from p-filters
% 
% This is still true, along with the proof:
%    \begin{lemma}
%      \label{lem:so-sd-implies-d}
%      Any filter that is both single-outputting and state-determined is
%      deterministic.
%    \end{lemma}
%    \begin{proof}
%      Following the procedure described in Lemma~\ref{lem:det-filts-are-pract} with a
%      single-outputting and state-determined filter never leads to any choices.
%      Therefore, only singleton subsets of $\pow{\set{V}}$ are involved. 
%      As a result, any observation terminal sequence can yield at most one
%      output.
%    \end{proof}
%
%  This is no longer true:
%
%    \begin{lemma}
%        All deterministic filters are single-outputting.
%    \label{lem:det-filts-are-so}
%    \end{lemma}
%
%  To see why, imagine a proof like this:
%    \begin{proof}
%      If $F$  were a deterministic
%      filter which was not single-outputting, there must be some sequence
%      arriving at a vertex in $\set{V_u}$ which, either has more than 
%      one departing edge, or the edge must have a label with at least two elements. 
%        <<Runs into the buffers>>
%    \end{proof}
%    You can have a single output, but from two separate edges with the same label.
%    This would be like a p-filter with a node which bears two colors. The first
%    color is red, and the second color is also red. You have a "choice" but not
%    really.
%
%
% Since we're aiming for the basis iff expression anyway that we can use to determine
% destructiveness, this is my solution:

\begin{theorem}
  \label{thm:so_is_enough_for_det_if_sd}
  Any state-determined p-graph is deterministic if and only if it is
  single-outputting.
\end{theorem}
\begin{proof}
  Forward direction:
If $F$ is a deterministic state-determined filter which is not
single-outputting, there must be some sequence arriving at a vertex in
$\set{V_u}$ which, either has more than one departing edge, or it must possess an edge 
with a label containing at least two elements. If it has more than one departing edge,
the labels cannot overlap because $F$ is state-determined. But either multiple
edges with distinct labels or an edge bearing a label with multiple elements imply
multiple successors, contradicting the requirement that $F$ be deterministic. \newline
The other direction: Following the procedure described in Lemma~\ref{lem:det-filts-are-pract} with a
single-outputting and state-determined filter never leads to any choices.
Therefore, only singleton subsets of $\pow{\set{V}}$ are involved. 
As a result, any observation terminal sequence can yield at most one
output.
\end{proof}

In the next section
%, specifically in Section~\ref{sec:det_filter_destructive}, 
 we use this result.

\subsection{Ascertaining destructiveness of observation maps on filters}

Since p-graphs are capable of representing filters, the next question is how
they might enable a roboticist to evaluate tentative designs and to better
understand solution space trade-offs. A class of interesting design-time
questions arises when one considers how modifications to a given robot's
capabilities alter the estimation efforts that the robot must undertake.  
%This section shows how the language of p-graphs can be used to address such
%questions.
In the specific context of filtering, consider the following illustrative
examples of how maps might come into play when we apply them to observation sets. 

\begin{example}
\label{ex:rose}
  Your robot is equipped with a camera, and triplets of red--green--blue values within
  an array comprise $\set{Y}$.  Now imagine that rose-tinted lenses are placed
  over the camera. Applied pixel-wise, $\func{h_{\textrm{rose}}}: \langle r, g,
  b\rangle \mapsto \{\langle r, 0, 0\rangle\}$. Certain scenes that produce distinct
  inputs, $y_1\neq y_2$, may now be indistinguishable under the transformation,
  as when
 two scenes differ only in elements of the spectrum filtered out by the
  lenses, and
  $\func{h_{\textrm{rose}}}(y_1) = \func{h_{\textrm{rose}}}(y_2)$.
\end{example}

\begin{example}
  Observation maps need not only reduce the set. Suppose your sensor incurs
  cross-talk due to poor cable routing and cheap shielding. Where formerly a
  given circumstance would produce an observation $y_i$, this might be modeled
  with an observation map \mbox{$y_i \mapsto \{y_i, y_i', y_i''\}$}.   It may
  be that $y' \in \set{Y}$, or it might be some heretofore unseen class of
  signal.  What we are interested in is whether this cross-talk is destructive
  or not. The answer to this depends on whether some other $y_j$
  exists where \mbox{$y_i' \in \func{h}(y_j)$}.  Even existence of such a $y_j$
  is insufficient, as $y_i$ and $y_j$ might occur in every pre-image together.
\end{example}

\newcommand{\eqmod}{\:\raisebox{0.01in}{\scalebox{0.8}{$\geqq$}}\:}

Next, we formalize the notion of a destructive observation map for filters. 

\begin{definition}[filter equivalence]
  \label{defn:filteq}
    Given two p-graphs for filtering,
      $F$ over event space $\set{E}=\set{U} \cup \set{Y}$ and
      $F'$ over event space $\set{E'}=\set{U} \cup \set{Y'}$, and
      an observation map $\func{h_y}: \set{Y} \to \pow{\set{Y}'} \setminus \{\emptyset\}$ mapping from the observation space of
      $F$ to sets of observations of $F'$, we say that \emph{
      $F$ is equivalent to $F'$ modulo $\func{h_y}$}, denoted
        \[
          {F} \eqmod {F'} \mod \func{h_y},
        \]
      if, for every observation-terminal sequence $e_1 e_2 e_3\cdots e_m$ in
      $\langof{F}$, 
      we have  that
      $\left\{ u \;\middle|\;  e_1 e_2 e_3\cdots e_m u \in \langof{F} \right\} = \\
             \phantom{}\hfill\left\{ v \;\middle|\; f_1 f_2 f_3 \cdots f_m v \in \langof{F'}, \mathop{\forall}\limits_{i} \left(e_i \in \set{Y}\implies f_i \in \func{h_y}(e_i)\right) \right\}.$

\end{definition}

Note that we eschew the traditional equivalence symbol `$\equiv$' for this
relation because it is not symmetric: $F \eqmod F' \mod \func{h_y}$ does not necessarily imply $F' \eqmod F \mod \func{h_y}$.

To understand the preceding condition, observe that, on $F$, the sequence $s =
e_1 e_2\cdots e_m$ produces an output that is an element of 
$\left\{ u \;\middle|\;  s u \in \langof{F} \right\}$, the 
set on the left-hand side. 
Paying attention to only the observations that comprise $s$,
which are every other element of the sequence,
each of these $e_i$ result in a set under $\func{h_y}$. 
We consider all sequences that have observations such that every
observation at position $i$ in the sequence, which we denote $f_i$, is from the
set $\func{h_y}(e_i)$.  Using all such sequences, we ask whether $F'$ produces
outputs that match $F$ on
$s$.

(Definition~\ref{defn:filteq} has a simpler presentation if we assume that the
filters involved have a finite action-space, in which case there exists a $k$,
for which we are permitted to write the set on the left as $\{u_1, u_2, \dots,
u_k\}$.  In what appears above we have not assumed that the action-space is
finite, nor even denumerable.)

The intuition behind the definition is that if $F'$, given observations mutated by $\func{h_y}$, exhibits the
same behavior that $F$ exhibits when given those same observations, but
unmutated, then any difference between $F$ and $F'$ is merely in
the change in manifestation of the observations that was induced by
$\func{h_y}$ and the underlying structure is the same.
In contrast, if the two filters can generate different outputs under these
conditions, then there must be some other explanation for those differences.
This suggestion motivates the idea of a nondestructive observation map.

\begin{definition}[non-destructive]
    \label{def:non-destructive-eq}
    Given a p-graph $F$ and 
      an observation map $\func{h_y}: \set{Y} \to \pow{\set{Y}'} \setminus \{\emptyset\}$, 
      we say that $\func{h_y}$ is {\em
    non-destructive} if
      $$ F \eqmod \func{h_y}(F) \mod \func{h_y}.$$
\end{definition}

Informally, a nondestructive observation map is one that preserves enough structure that
the filter still works after applying it, as long as the labels are updated
accordingly.  A destructive observation map is one that
creates enough ambiguity (initially expressed in the resulting p-graph by
states with out-edges whose labels overlap) that the correct outputs can no longer be
determined solely by the observations.

\begin{algorithm}
  \caption{{\sc EquivalenceModuloMap}$({F_1}, {F_2}, \func{h_y})$}
  \label{alg:equiv}
  \DontPrintSemicolon
  \SetAlgoLined
  \SetKwInOut{Input}{Input}
  \SetKwInOut{Output}{Output}
  \Input{Two finite action-space p-graphs ${F_1}$ and ${F_2}$, and observation map $\func{h_y}: \set{Y} \to \pow{\set{Y}'} \setminus \{\emptyset\}$.}
  \Output{True iff $F_1 \eqmod F_2 \mod \func{h_y}$}
  \BlankLine
  \If{$F_1$ and $F_2$ are not akin}{
      {\textbf{return} False}
  }
  {Convert ${F_1}$ and ${F_2}$ to state determined presentation if needed.}\;
  \tcc{Basic idea: conduct a forward search, computing the finite-set of observations needed to make all potential transitions along the way.}\;
  {Initialize queue ${Q} \leftarrow \set{V^{(F_1)}_0} \times \set{V^{(F_2)}_0}$} \;
      \While{$Q$ is not empty}{
          $(s_1,s_2) \leftarrow Q.\textrm{pop}$\;
          \If{$s_1$ and $s_2$ are observation vertices}{
              $\set{Y_1} \leftarrow \textsc{RefineLabels}(\mbox{labels leaving}~s_1)$ \;
              $\set{Y_2} \leftarrow \textsc{RefineLabels}(\mbox{labels leaving}~s_2)$ \;
              $\set{Y_2'} \leftarrow \{$ pre-image of each element of $\set{Y_2}$ under $\func{h_y}\}$\;
              $\set{L}  \leftarrow \textsc{Representatives}(\set{Y_1} \cup \set{Y_2'})$\;
              \tcc{$\set{L}$ has a partition of the observation space which is just fine enough to exercise each filter.}
              \For{$\ell \in \set{L}$}{
                  {$s_1' \leftarrow $ state that ${F_1}$ transitions to on $\ell$}\;
                  {$s_2' \leftarrow $ state that ${F_2}$ transitions to on $\func{h_y}(\ell)$}\;
                  {$Q.\textrm{push}((s_1',s_2'))$} \tcp{To be processed}
              }
          }
          \Else{\tcp{Both are action vertices}
              $\set{O_1} \leftarrow \textsc{Union}(\mbox{labels leaving}~s_1)$ \;
              $\set{O_2} \leftarrow \textsc{Union}(\mbox{labels leaving}~s_2)$ \;
              {\tcp{Equality of label sets computed using \textsc{Difference} and \textsc{Empty}}}
              \If{$\set{O_1} \neq \set{O_2}$}{
                {\textbf{return} False}~\tcp{Output sets are not equal}
              }
              $\set{U_1} \leftarrow \textsc{RefineLabels}(\mbox{labels leaving}~s_1)$ \;
              $\set{U_2} \leftarrow \textsc{RefineLabels}(\mbox{labels leaving}~s_2)$ \;
              $\set{L}  \leftarrow \textsc{Representatives}(\set{U_1} \cup \set{U_2})$\;
              \For{$\ell \in \set{L}$}{
                  {$s_1' \leftarrow $ state that ${F_1}$ transitions to on $\ell$}\;
                  {$s_2' \leftarrow $ state that ${F_2}$ transitions to on $\ell$}\;
                  {$Q.\textrm{push}((s_1',s_2'))$} \tcp{To be processed}
              }
          }
      }
  {\textbf{return} True}\;
\end{algorithm}

\begin{example}
\label{example:monot}
  Suppose $\func{h_y}$ %: \set{Y} \to \pow{\set{K}} \setminus \{\emptyset\}$
  is an injective map so that if
  $\func{h_y}(y) = \func{h_y}(z)$, then $y=z$.  Because this kind of map does
  not introduce the possibility of conflating any two observations, it is clear
  that $\func{h_y}$ is non-destructive.
  In the particular case of interval labels (recall Example~\ref{ex:interval}),
  this implies that any sensor map that is a strictly-increasing or
  strictly-decreasing\,---including, for example, affine maps---\,is
  non-destructive.
  Contrapositively, we can also conclude that every destructive sensor map is
  non-injective.
\end{example}

Restricting ourselves to finite
action-spaces, 
we can now pose the problem posed by the examples in Section~\ref{sec:intro}
precisely and address it algorithmically. 
Given a p-graph $F$ and an
observation map $\func{h_y}: \set{Y} \to \pow{\set{Y}'} \setminus
\{\emptyset\}$, we wish to determine if $\func{h_y}$ is non-destructive on $F$
or not. We check explicitly whether $F \eqmod \func{h_y}(F) \mod \func{h_y}$.
Algorithm~\ref{alg:equiv} shows how to perform this check.  After converting to
state-determined normal form, if necessary, the algorithm uses a forward search
over pairs of states, one from each p-graph, that are reachable by some
event sequence.  For each such
pair, if they are action vertices we verify that the outputs specified by the p-graphs are the same.
For full generality, we show the algorithm for arbitrary pairs of p-graphs, not
just for an $F$ and its $\func{h_y}(F)$.

\subsubsection{Deterministic filters}
\label{sec:det_filter_destructive}

Now suppose we have a deterministic filter $F$ and an observation map $\func{h_y}$,
and wish to ascertain whether $\func{h_y}$ is non-destructive on $F$---a special
case that should be quite common, since  those which are directly
implementable, \emph{viz.} practicable filters, are deterministic filters.
For these filters we can use
Algorithm~\ref{alg:expansion} along with
Theorem~\ref{thm:so_is_enough_for_det_if_sd} to determine whether $\func{h_y}$
is destructive.  The intuition is to compute $\func{h_y}(F)$, then convert that
mapped filter to a state-determined presentation and check whether the result
is also single-outputting. Checking whether a filter output from
Algorithm~\ref{alg:expansion} is single-outputting is especially straightforward because
it only outputs reachable
vertices, thus, one simply checks that 
each action vertex has at most
one singleton labeled edge departing it.
Algorithm~\ref{alg:smdt-d} gives the overall test for destructiveness, which is strikingly simple. 

\begin{algorithm}
  \caption{\textsc{ObservationMapDestructivenessTest}$(F, \func{h_y})$}
  \label{alg:smdt-d}
  \DontPrintSemicolon
  \SetAlgoLined \SetKwInOut{Input}{Input} \SetKwInOut{Output}{Output}
  \Input{A deterministic filter $F$ and an observation map~$\func{h_y}$.}
  \Output{\textsc{True} iff $\func{h_y}$ is non-destructive on $F$.}
  \BlankLine
  $G \leftarrow \textsc{ToStateDeterminedPresentation}(\func{h_y}(F))$ \;
  \textbf{return} $\textsc{IsSingleOutputting}(G)$ 
\end{algorithm}

\subsection{Hardness}
\label{sec:hardness}

The preceding treatment of observation maps raises the question of
why it is of interest to consider a variety of maps.  Instead, why not
simply find the observation map that is, in some sense, the `most aggressive'
nondestructive map for a given filter?  In this section, we present a hardness
result establishing that, unless $P=NP$, no efficient algorithm can find the
nondestructive sensor map of minimal image size for a given filter, even
approximately.
Specifically, we consider the following decision problem:

\begin{definition}[sensor minimization]\label{def:sm}
The \emph{sensor minimization decision problem} is: 
Given a p-graph $F$ and integer $n$, return \textsc{True} if there exists a set
$\set{K}$ and an observation map 
$\func{h_y}: \set{Y} \to \pow{\set{K}} \setminus \{\emptyset\}$, nondestructive for
$F$, with $|\set{K}|\le n$, and \textsc{False} otherwise.
\end{definition}

\begin{figure}
  \centering
  \includegraphics[scale=0.65]{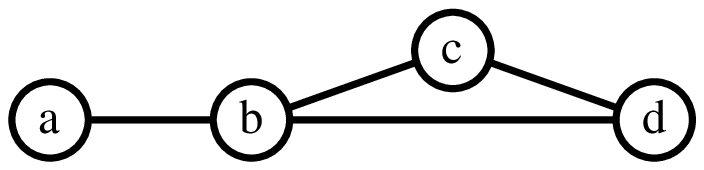}\qquad
  \includegraphics[scale=0.65]{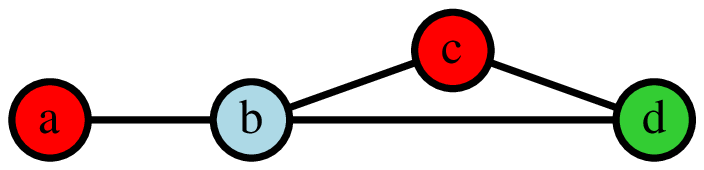}
  \caption{[left] An example instance of the 3-coloring problem.  [right] A coloring of
  that graph using three colors.}
  \label{fig:hardness-gc}
\end{figure}

\begin{theorem}\label{thm:sm-hard}
The sensor minimization decision problem is NP-hard.
\end{theorem}

\begin{figure}
  \centering

  \includegraphics[width=\columnwidth]{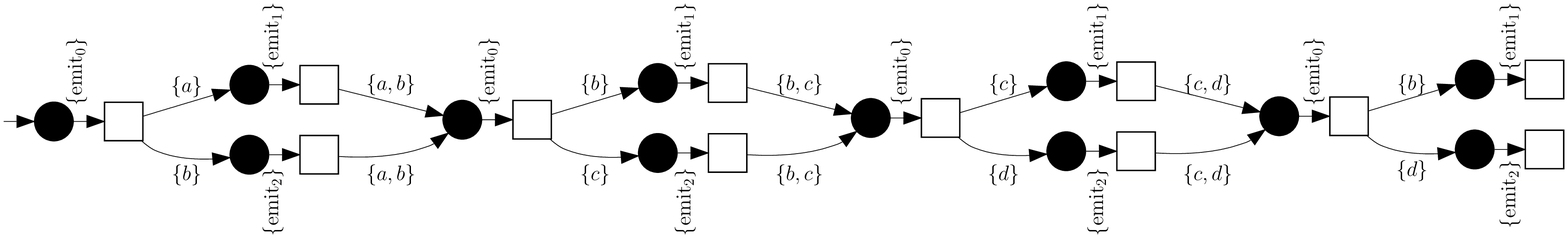}

  \bigskip

  \includegraphics[width=\columnwidth]{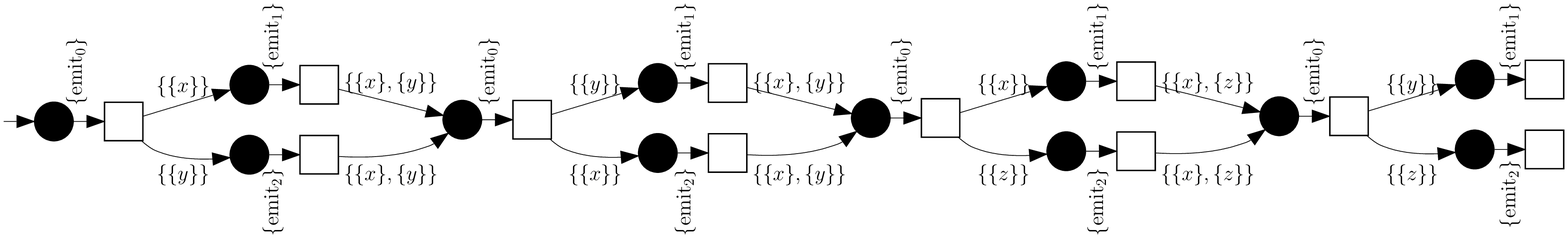}
  \caption{[top] The p-graph expressing a filter constructed from the graph coloring problem 
  shown in Figure~\ref{fig:hardness-gc}. [bottom] The result of applying an observation
  map under which $a \mapsto \{x\}$, $b \mapsto \{y\}$, $c \mapsto \{x\}$, and $d \mapsto
  \{z\}$.  This mapped filter is equivalent to the original filter, modulo this
  map.  Because this filter has a non-destructive map with image of size 3, the
  graph in
  Figure~\ref{fig:hardness-gc} can be colored with 3 colors.}
  \label{fig:hardness-filter}
\end{figure}

\begin{proof}
  Reduction from the graph 3-coloring problem \textsc{Graph-3c}, which is known
  to be NP-complete~\citep{GareyJohnson1979}.  Given an instance $G$ of
  \textsc{Graph-3c}, we construct an instance of the sensor minimization
  decision problem, building an action-first p-graph
  $F$ as follows:
  Use one observation in $\set{Y}$ for each vertex of $G$, so that $\set{Y} = \set{V}(G)$.
  For the set of actions, $\set{U}$, select $\{{\rm emit}_{0}, {\rm emit}_{1}, {\rm emit}_{2}\}$.
  Assign an arbitrary but fixed ordering to the edges $\set{E}(G)$.
  For each edge $e \in \set{E}(G)$ connecting nodes $v$ and $w$: 
     (1)\,insert three action vertices $\act{i}_e$, $\act{s}_e$, and $\act{t}_e$, into $\set{\Vu}$;
     (2)\,insert three observation vertices $\obs{i}_e$, $\obs{s}_e$, and $\obs{t}_e$, into $\set{\Vy}$.  Here, the names mnemonically indicate
     `initial,' `source node,' and `target node.'
  Continue to build the p-graph $G$ by adding an edge from $\act{i}_e$ to $\obs{i}_e$
  labeled with output $\{{\rm emit}_{0}\}$, adding an edge from $\act{s}_e$ to $\obs{s}_e$
  labeled with output $\{{\rm emit}_{1}\}$, and adding an edge from $\act{t}_e$ to $\obs{t}_e$
  labeled with output $\{{\rm emit}_{2}\}$.
  Add an edge labeled with the observation set $\{ v \}$ from $\obs{i}_e$ to $\act{s}_e$.
  Likewise, add an edge labeled with the observation set $\{ w \}$ from
  $\obs{i}_e$ to $\act{t}_e$.
  Unless $e$ is the final edge in the ordering, let $e'$ denote the next
  $G$-edge in the arbitrary ordering, and add edges to $F$ labeled $\{ v, w \}$
  from $\obs{s}_e$ to $\act{i}_{e'}$ and from $\obs{t}_e$ to $\act{i}_{e'}.$
  For the first edge $e \in \set{E}(G)$ in the ordering, designate $\act{i}_e$ as the
  single initial node $\set{\Vinit}$.
  Select $n=3$.

  Figure~\ref{fig:hardness-gc} shows an example instance of \textsc{Graph-3c},
  and Figure~\ref{fig:hardness-filter} shows the corresponding filter.
  The construction takes time linear in the size of $G$.
  It remains to show that $G$ is 3-colorable if and only if $F$ has a
  nondestructive observation map $\func{h_y}: \set{Y} \to \pow{\set{K}}$
  with $|\set{K}|\le 3$.

  Assume that $G$ is 3-colorable.  Let $\func{c}: V(G) \to \{0, 1, 2\}$ be a 3-coloring
  of $G$.  
  Since $\set{Y} = \set{V}(G)$, we construct an observation map for $F$ using $c$
  as follows. We let $\set{K} = \{0, 1, 2\}$ and map only to singleton subsets: $\func{h_y}(y) \mapsto \{c(y)\}$.

  Note that $\func{h_y}(F)$ is state-determined; the only states with multiple out-edges are
  the $\obs{i}_e$ states, and since $c$ is a coloring of $G$, the two observations
  labeling these edges in $F$ must map to different sets under $\func{h_y}$.
  Therefore, $\func{h_y}$ is a nondestructive sensor map for $F$, and $|\set{K}|=3$.

  For the other direction, assume $F$ has a nondestructive observation map
  $\func{h_y}: \set{Y} \to \pow{\set{K}} \setminus \{\emptyset\}$ with $|\set{K}|=3$. Then, there must
  also exist a nondestructive map $\func{h^{s}_y}$ mapping to singletons from $\set{K}$, i.e.,
  $\left\{\{k\}\;\middle|\;k \in \set{K}\right\}$ since mapping to sets with more
  than one element only loses information. From $\func{h_y}$ we can construct an
  $\func{h^{s}_y}$ by making some arbitrary choice from items in the image set.
  
  We argue that this $\func{h^{s}_y}$ forms a valid 3-coloring of $G$.
  Suppose, to the contrary, that $\func{h^{s}_y}$ is not a valid 3-coloring of
  $G$.  Then there must exist some edge $e \in \set{E}(G)$, connecting two
  nodes $v$ and $w$, such that $\func{h^{s}_y}(v)=\func{h^{s}_y}(w)$.  But in
  that case, in $h^{s}_y(F)$, and hence $h_y(F)$, from the node $\obs{i}_e$
  there are two out-edges, both intersecting labels, leading to
  differently-colored states, namely $\act{s}_e$ and $\act{t}_e$.  But
  $\act{s}_e$ results in $\{{\rm emit}_{1}\}$, while $\act{t}_e$ does an $\{{\rm
  emit}_{2}\}$.
   In contrast, the original $F$
  is state-determined.  Therefore $\func{h^{s}_y}$ is destructive of $F$, a contradiction.

  Finally, since \textsc{Graph-3c} is polynomial-time reducible to sensor minimization, we
  conclude that it is NP-hard.
\end{proof}

Note, {\em a fortiori}, that the proof of Theorem~\ref{thm:sm-hard} does not
depend any essential way on the specific number~3.  In fact the chromatic
number of the graph coloring instance and the image size of the smallest
nondestructive observation map for the corresponding p-graph filter are always
equal.  Combined with known results on the inapproximability of chromatic
numbers~\citep{Zuckerman07}, this leads directly to the following stronger
result.

\begin{corollary}
  The optimization problem of finding, for a given filter, the nondestructive
  sensor map with the smallest image size, is NP-hard to approximate to within
  $n^{1-\epsilon}$.
\end{corollary}
\begin{proof}
  Let $\epsilon > 0$.
  Suppose that there exists a  polynomial time approximation algorithm $A$ 
  to solve 
  sensor minimization
  with approximation ratio $n^{1-\epsilon}$.
  Let $B$ denote an approximation algorithm for graph coloring that works as
  follows.
  \begin{enumerate}
    \item Given an instance $G$ of graph coloring problem, form the filter
      $F$ as described in the proof of Theorem~\ref{thm:sm-hard}.
    \item Use algorithm $A$ to apply map $\func{h_y}$ on $F$.

    \item Find a coloring of $G$ from the applied map $\func{h_y}(F)$.
  \end{enumerate}
  We now argue that $B$ has approximation ratio $n^{1-\epsilon}$.
 
  Let $B(G)$ denote the number of colors used by algorithm $B$ to color $G$,
  and, likewise,
  let $A(\func{h_y}(F))$ denote the image size of filter produced
  by algorithm $A$ from input filter $F$ and with map $\func{h_y}$ applied.
  Let $OPT(\func{h_y}(F))$ and $OPT(G)$ represent the smallest image size of applying $\func{h_y}$
  on filter $F$ and minimum number of colors for coloring $G$, respectively.
  According to the assumption, we have $A(F)\leqslant n^{1-\epsilon} OPT(\func{h_y}(F))$.
  \newline
  Thus, the above construction would be such an approximation algorithm $B$ for
  graph coloring problem. So, we have $OPT(\func{h_y}(F)) = OPT(G)$. Then, for 
  sufficiently large $n$,
   \begin{eqnarray}
         B(G)
         &=&   A(\func{h_y}(F(G)))  \label{eq:2} \\
         &\leqslant& n^{1-\epsilon}  OPT(\func{h_y}(F))  \label{eq:3} \\
         &\leqslant& n^{1-\epsilon}  OPT(G)  \label{eq:4} .
   \end{eqnarray}
  Therefore, $B$ is an approximation algorithm for graph coloring with
  approximation ratio $n^{1-\epsilon}$, which contradicts
  \citet{Zuckerman07}.

\end{proof}

\subsection{Case study: Minimizing the iRobot Create}
\label{sec:casestudy}

\begin{figure}[h!]
  \centering
  \includegraphics[width=0.4\columnwidth]{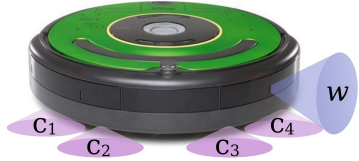}
  \caption{An iRobot Create is equipped with a collection of simple sensors
  including four cliff sensors and a wall sensor, each uses IR to measure
  distance. As a simple example, we consider a filter which maps sensor
  readings into motor commands on a robot tasked with following a wall on its
  left, while avoiding negative obstacles. (A suitable environment is shown in Figure~\ref{fig:pgraph-example}.)}
  \label{fig:roomba}
\end{figure}

The following simple scenario, of the sort that the authors have often assigned
in introductory robotics courses, illustrates the utility of the machinery
developed in this paper. 
\textcolor{black}{
Here, we report transformations computed by
our Python implementation of the algorithms described above.}
Revisiting the scenario in
Figure~\ref{fig:pgraph-example}, we wish to have an iRobot Create vacuum
cleaning robot follow walls (on its port side) while avoiding negative
obstacles. The five range sensors on the robot provide sufficient information
to carry out this basic task.  (See Figure~\ref{fig:roomba} for elaboration of
the sensing details.) We approach this problem by constructing a p-graph whose
outputs map directly to actions for the robot, and then we are able to analyze
the effect of observation maps on this p-graph (as a filter) with our
implementation of the algorithms described in the earlier sections of the
paper.  

\begin{figure}[h!]
  \centering
  \includegraphics[width=0.8\columnwidth]{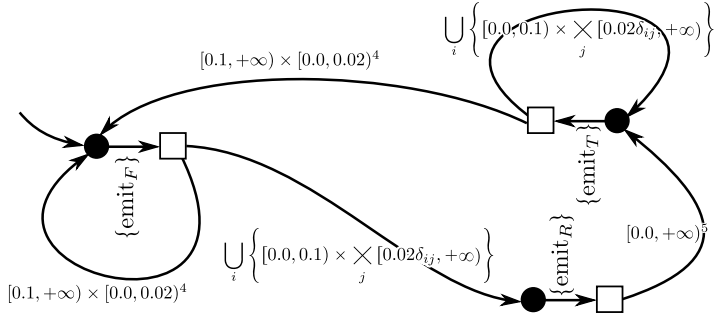}
  \caption{A visual representation of the filter ({\sc Ideal}) that solves navigation problem for the Create  where
  edge labels are subsets of $\real^5$, and the values in each vertex are velocities that the robot executes
  for some small finite time. The action ${\rm emit}_{F}$ generates a `Forward-with-slight-left-bias' motion with $\dot{x} = 0.2, \dot{\theta} = 0.1$,
  the ${\rm emit}_{T}$ generates a `Turn' motion via $\dot{x} = 0, \dot{\theta}=-0.2$, and 
  the ${\rm emit}_{R}$ generates a `Reverse' motion via $\dot{x} = -0.1, \dot{\theta}=0$.
(The $\delta_{ij}$ is the Kronecker delta, $i, j \in \{1,2,3,4\}$.)
  }
  \label{fig:roombafilter}
\end{figure}

First we describe the set of observations for idealized versions of the robot's
sensors. Each of $\{w, c_1, c_2, c_3, c_4\}$ is fundamentally a device that
measures distance, so it is useful to model each output with a real number that
represents the range reading; naturally, the product of these five sensors
gives a label space with $\real^5$. Each state in the filter produces an output
that is interpreted as velocity commands---linear as $\dot{x}$ and angular as
$\dot{\theta}$. The filter is shown pictorially in
Figure~\ref{fig:roombafilter}.

From this a series of filters is constructed via transformations that coarsen
the label space. Observation map $\func{h_y}$ clips to a maximum ranges
supported by the sensors. It is applied to the {\sc Ideal} filter, giving a filter
{\sc Create (Signals)}, whose labels are based on the data that can be read
from the physical sensors through the software interface (see \citet{oi}).
Observation map $\func{f_y}$, transforms {\sc Create
(Signals)} into {\sc Create (Symbols)} representing a second level of
abstraction ---in this case, a quantization based on thresholding--- available
through the robot's hardware interface. Map $\func{g_y}$ further reduces the set
of labels, while the final map we define, $\func{k_y}$, is destructive.
Table~\ref{tab:filts} collects this information.  The rows in the table also
summarize the relationships visually. Starting from {\sc Ideal} one produces
the others via composition of the sensor maps, for example, {\sc Combined
Sensor} results from applying map the~$\func{g_y}\circ\func{f_y}\circ\func{h_y}$.
Only the final map to a sensorless model is destructive. The conclusion 
we reach from this is
that, for this filter, neither the specific distance measurements nor
the
individual identities of the sensors themselves are necessary.  A robot
designed exclusively for this task could, therefore, likely be designed to be
simpler than a Create.

\begin{table}
\noindent
\begin{center}
\begin{tabular}{|llp{0.32\linewidth}|}
\hline
\bf{Name              } &{\bf   Observation Space} &{\bf     Notes}\\[2pt]
\hline &&\\[-5pt]
{\sc Ideal}              &     {\footnotesize$\real\times\real\times\real\times\real\times\real$} & \\[6pt]
\textcolor{darkred}{\bf\quad$\Big\downarrow\func{h_y} = \textrm{clipToRange}(\cdot)$} & &\\ [9pt]
{\sc Create (Signals)}   &     {\footnotesize$[0,1023]\times[0,4095]^4$} & {\scriptsize {\em cf.}~\citet[pg.~27]{oi}.}\\[6pt] % Pretending these aren't integers
\textcolor{darkred}{\bf\quad$\Big\downarrow\func{f_y} = \textrm{threshold}_T(\cdot)$}  & \multicolumn{2}{l|}{\textcolor{darkred}{\quad($T=10$ and $20$ for $w$ and $c_i$ respectively.)  }} \\[9pt]  
{\sc Create (Symbols)}   &     {\footnotesize$\{0,1\}\times\{0,1\}^4$} & {\scriptsize {\em cf.}~\citet[pgs.~22--23]{oi}.}\\[6pt]
\textcolor{darkred}{\bf\quad$\Big\downarrow\func{g_y} = \textrm{min}(\cdot)$}  & & \\[9pt]
{\sc Combined Sensor}    &     {\footnotesize$\{0,1\}$}  & \\[6pt]
\textcolor{darkred}{\bf\quad$\Big\downarrow\func{k_y} = 0$}  & \multicolumn{2}{l|}{\textcolor{darkred}{\quad(Constant map)  }} \\[9pt]
{\sc Sensorless}        &     {\footnotesize$\{0\}$}    & (Destructive)\\[12pt]
\hline
\end{tabular}
\end{center}
\caption{A hierarchy of filters for the iRobot Create.\label{tab:filts}}
\end{table}

\begin{comment}
% I don't think we want these details. If we do, they need to be updated to output sets

\[\func{h_y}(\langle x_w,x_{c_1},\dots, x_{c_4}\rangle)= \langle\func{h_w}(x_w),\func{h_c}(x_{c_1}),\dots,\func{h_c}(x_{c_4})\rangle\]

\[h_w(x)= \begin{cases}
0, & x \leq 0 \\
100x,& 0 < x \leq 10.23\\
1023,& \textrm{otherwise.}
\end{cases} \]

\[h_c(x)= \begin{cases}
0, & x \leq 0\\
1000x, & 0 < x \leq 4.095\\
4095, & \textrm{otherwise.}
\end{cases} \]

\[\func{f_y}(\langle x_w,x_{c_1},\dots, x_{c_4}\rangle)= \langle\func{f_w}(x_w),\func{f_c}(x_{c_1}),\dots,\func{f_c}(x_{c_4})\rangle\]

\[f_w(x)= \begin{cases}
1, & x < 10 \\
0, & \textrm{otherwise.}
\end{cases} \]

\[f_c(x)= \begin{cases}
0, & x < 20 \\
1, & \textrm{otherwise.}
\end{cases} \]

\[\func{g_y}(\langle x_w,x_{c_1}, \dots\rangle)= 
\begin{cases}
 1, & {(x_w+x_{c_1}+\dots+x_{c_4}) > 1}\\
 0, & \textrm{otherwise.}
\end{cases} \]
\end{comment}

\section{Destructiveness in planning problems}\label{sec:plans}
\label{sec:plan}

\subsection{Plans and Planning Problems}

The final section of \citet{TovCoh+14}, a substantial and recent paper on the
topic of combinatorial filters, concludes with
the following:
\begin{quotation}%
\noindent\textsl{``Since the methods so far provide only inference, how can their output be used
to design motion plans? In other words, how can the output be used as a filter
that provides feedback for controlling how the bodies move to achieve some
task?''}
\end{quotation}
Next, we make some progress in that direction by using p-graphs to model
planning problems and plans.  The example of the iRobot Create in
Section~\ref{sec:casestudy} had a direct correspondence between filter outputs and
actions that the robot executed. In general, the sequences of
actions that a robot performs will depend on the task it is performing, but in the previous
section there was no direct representation of tasks.  
Thus, though p-graphs have been used up to this point to encode state space
structure, more information must be provided to talk meaningfully about
plans and planning problems. 

\begin{definition}[planning problem]\label{def:problem}
  A \emph{planning problem} is a p-graph $G$ equipped with a \emph{goal region}
  $\Vgoal \subseteq V(G)$.
\end{definition}

The idea is that for a pair that make up the planning problem, the p-graph
describes the setting and form in which decisions must be made, while the
$\Vgoal$ characterizes what must be achieved.  Recall that, because a p-graph
may have multiple initial states, this definition can encompass planning
problems in which the system is known to start from one of possibly many
starting states.

\begin{definition}[plan]\label{def:plan}
  A \emph{plan} is a p-graph $P$ equipped with a \emph{termination region}
  $\Vterm \subseteq V(P)$.
\end{definition}

The intuition is that the out-edges of each action state of the plan show one or
more actions that may be taken from that point---if there is more than one such
action, the robot selects one nondeterministically---and the out-edges of each
observation state show how the robot should respond to the observations received
from the environment.  If the robot reaches a state in its termination region, it
may decide to terminate there and declare success, or it may decide to continue
on normally.
This, then, gives an interpretation of p-graphs as plans.
We can now establish the core relationship between planning problems and plans.
%TODO: Include something explicit about V_0 being the start states.

\begin{definition}[solves]\label{def:solves}
  A plan $(P, \Pterm)$ \emph{solves} the planning problem $(W, \Vgoal)$ if
  $P$ is finite and safe on $W$, and every joint-execution $\eseq$ of
  $P$ on $W$
  either reaches a vertex in $\Pterm$, or is a prefix of some joint-execution
  that reaches $\Pterm$ and, moreover, all the $\eseq$ that reach a vertex
  $v\in V(P)$ with $v \in \Pterm$, always reach vertices $w\in V(W)$ with $w \in
  \Vgoal$.
\end{definition}

% To check: a p-graph P which has no start state has no executions, so has no
% joint-executions. That means every joint execution can satisfy whatever
% requirement you desire. Solution: the empty execution is an execution, so 
% this must be a prefix, etc.

\textcolor{black}{
The solution concept here, with its stipulation of finiteness 
reminiscent of notions of computability, is concerned only with processes that terminate
in some bounded time. (We defer questions about extensions to other prevalent
concepts\,---such as infinite horizons, models of rewards, and the like---\,to
future work.) }

% A useful example to show this definition working hard involves a plan that can achieve one of two goals, the first part being a prefix of the second.

\begin{example}[Charging around and in the pentagonal world]
  We can construct a planning problem from the p-graph of
  Figure~\ref{fig:pentaworld}, along with a goal region consisting of only
  the fully-charged state reached by action $u_2$.  Figure~\ref{fig:35} shows a
  plan that solves this problem.
  \begin{figure}[t]
    \begin{minipage}[t]{0.5\textwidth}
      \centering
      \includegraphics[scale=0.5]{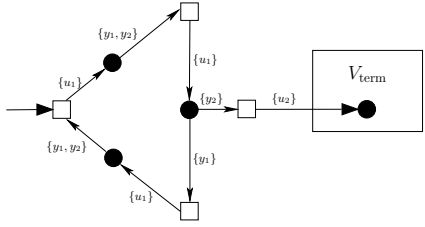}
    \end{minipage}
    \begin{minipage}[b]{0.5\textwidth}
      \caption{A plan that directs the robot of
      Figure~\ref{fig:pentaworld} to its charging station, along a
      \hbox{hyperkinetic} (that is, exhibiting more motion than is strictly
      necessary) path.\label{fig:35}}
    \end{minipage}
  \end{figure}
  However, that plan, a cycle of three actions, is a bit surprising since it
  will take the robot along three full laps around its environment before
  terminating.
  \gobble{Each of the joint executions in this case are prefixes of the event sequence
    $u_1 y_1 u_1 y_1 u_1 y_1 u_1 y_2 u_1 y_1 u_1 y_1 u_1 y_1 u_1 y_1 u_1 y_2 u_1 y_1 u_1 y_1 u_1 y_1 u_1 y_1 u_1 y_2 u_2 $.}
  The existence of such bizarre plans motivates our consideration of
  homomorphic plans, which behave rather more sensibly, in
  Section~\ref{sec:homomorophicsolns}.
  %
  %
  % \begin{tabular}{cccc}
  %  $i$ & $e_i$ & in $W$ & in $P$ \\
  %  0   &       & a1     & a1     \\
  %  1   & $u_1$ & o1     & o1     \\
  %  2   & $y_1$ & a2     & a2     \\
  %  3   & $u_1$ & o2     & o2     \\
  %  4   & $y_1$ & a3     & a3     \\
  %  5   & $u_1$ & o3     & o3     \\
  %  6   & $y_1$ & a4     & a1     \\
  %  7   & $u_1$ & o4     & o1     \\
  %  8   & $y_2$ & a5     & a2     \\
  %  9   & $u_1$ & o5     & o2     \\
  %  10  & $y_1$ & a1     & a3     \\
  %  11  & $u_1$ & o1     & o3     \\
  %  12  & $y_1$ & a2     & a1     \\
  %  13  & $u_1$ & o2     & o1     \\
  %  14  & $y_1$ & a3     & a2     \\
  %  15  & $u_1$ & o3     & o2     \\
  %  16  & $y_1$ & a4     & a3     \\
  %  17  & $u_1$ & o4     & o3     \\
  %  18  & $y_2$ & a5     & a1     \\
  %  19  & $u_1$ & o5     & o1     \\
  %  20  & $y_1$ & a1     & a2     \\
  %  21  & $u_1$ & o1     & o2     \\
  %  22  & $y_1$ & a2     & a3     \\
  %  23  & $u_1$ & o2     & o3     \\
  %  24  & $y_1$ & a3     & a1     \\
  %  25  & $u_1$ & o3     & o1     \\
  %  26  & $y_1$ & a4     & a2     \\
  %  27  & $u_1$ & o4     & o2     \\
  %  28  & $y_2$ & a5     & aT     \\
  %  29  & $u_2$ & oG     & oT     \\
  %  
  % \end{tabular}
  \label{ex:35}
\end{example}

\textcolor{black}{
Both plans and planning problems are pairs consisting of a p-graph and a set of
states. In each case, the p-graph can be converted into a state-determined
presentation using Algorithm~\ref{alg:expansion} and doing so preserves the
interaction language.  But the semantics for both structures,
plans and planning problems, is tied together through the definition of
`solves' (Definition~\ref{def:solves}).  That definition has two parts. The
first concerns finiteness and safety, properties of joint-executions only, and
is consequently unaffected by transformations that preserve the interaction
language.  The second depends on vertices and their relationship to the
associated sets.  Thus, forming something analogous to a state-determined
presentation must require some alteration of the second element of the pair,
i.e., the set of states.  
}

\textcolor{black}{
\begin{definition}[state-determined planning problems] \label{def:stdprob}
The state-determined presentation of planning problem $(W, \Vgoal)$
is $(W', \Vgoalp)$ where $W'$ is the 
state-determined presentation of p-graph $W$, and $\Vgoalp$ is the
subset of $V(W')$ where $v \in V(W')$ is included in $\Vgoalp$ only 
if all the vertices in $V(W)$ that correspond with $v$ are in $\Vgoal$.
\end{definition}%
}

\textcolor{black}{
\begin{definition}[state-determined plans] \label{def:stdplan}
The state-determined presentation of problem $(P, \Vterm)$
is $(P', \Vtermp)$ where $P'$ is the 
state-determined presentation of p-graph $P$, and $\Vtermp$ is the
subset of $V(P')$ where $v \in V(P')$ is included in $\Vtermp$ only 
if there exists a vertex in $V(P)$ corresponding with $v$ that is in $\Vterm$.
\end{definition}%
}

\textcolor{black}{
The previous two definitions differ only in terms of the quantifier involved in
their conditions on associated vertices. 
These mirror the `reach a vertex' and `always reach vertices' in 
Definition~\ref{def:solves}.
Practically, the conditions can be
computed easily by via the $\set{\textrm{Corresp}}[\cdot]$ map used in
Algorithm~\ref{alg:expansion}.  
}

\textcolor{black}{
\begin{lemma}[state-determined presentations preserve solubility]
If $(P, \Vterm)$ is a plan, and $(W, \Vgoal)$ a planning problem, with
their state-determined presentations being 
$(P', \Vtermp)$ and $(W', \Vgoalp)$ respectively, then the following 
are equivalent:
\begin{enumerate}
\item $(P, \Vterm)   \textrm{ solves } (W, \Vgoal)$
\item $(P', \Vtermp) \textrm{ solves } (W, \Vgoal)$
\item $(P, \Vterm)   \textrm{ solves } (W', \Vgoalp)$
\item $(P', \Vtermp) \textrm{ solves } (W', \Vgoalp)$
\end{enumerate}
\end{lemma}
\begin{proof}
%$1. \implies 2.$: 
%TODO Check details carefully
The correspondence between vertices in the original p-graph
and the state-determined presentations allows the joint-executions
in one case to be traced in the other.
Under the correspondence, one must check the requirements for
being a solution do in fact hold.
But the logic necessary in updating the goal and termination sets in
Definitions~\ref{def:stdprob} and~\ref{def:stdplan}
correspond to the solution requirements (an event sequence reaching \emph{a} vertex
in $\Vterm$ will reach a set of vertices \emph{all} of which are in $\Vgoal$),
so they do hold.
\end{proof}
}

Now, given a plan $(P, \Pterm)$ and a planning problem $(W, \Vgoal)$, we can decide
whether $(P, \Pterm)$ solves $(W, \Vgoal)$ in a relatively straightforward way.
%Old text: First, we convert both $P$ and $W$ into state-determined presentations, using Algorithm~\ref{alg:expansion}.  
\textcolor{black}{
First, we convert both into state-determined presentations, as just described.%
}%
Then, the algorithm
conducts a forward search using a queue of ordered pairs $(v,w)$, in which $v
\in V(P)$ and $w \in V(W)$, beginning from the (unique, due to
Definition~\ref{defn:sd}) start states of each.  For each state pair $(v, w)$
reached by the search, we can test each of the properties required by
Definition~\ref{def:solves}:
\begin{itemize}
  \item If $P$ and $W$ are not akin, return false.
  \item If $(v, w)$ has been visited by the search before, then we have
  detected the possibility of returning to the same situation multiple times in
  a single execution.  This indicates that $P$ is not finite on $W$.  Return
  false.

  \item If $v$ and $w$ fail the conditions of Definition~\ref{def:safe} (that
  is, if $v$ is missing an observation that appears in $w$, or $w$ omits an
  action that appears in $v$) then $P$ is not safe on $W$.  Return false.

  \item If $v$ is a sink state  not in $\Pterm$, or $w$ is a sink state not in
  $\Vgoal$, then we have detected an execution that does not achieve the goal.
  (A vertex is a sink if it has no departing edges.)
  Return false.

  \item If $v \in \Pterm$ and $w \notin \Vgoal$, then the plan might terminate
  outside the goal region.  Return false.
\end{itemize}
If none of these conditions hold, then we continue the forward search, adding
to the queue each state pair $(v', w')$ reached by a single event from $(v,w)$.
Finally, if the queue is exhausted, then---knowing that no other state pair can
be reached by any execution---we can correctly conclude that $(P, \Pterm)$ does
solve $(W, \Vgoal)$.

It may perhaps be surprising that both planning problems and plans are defined
by giving a p-graph, along with a set of states at which executions should end.
We view this symmetry as a feature, rather than a bug, in the sense that it clearly
illuminates the duality between the robot and the environment with which it
interacts.  As alluded to in Section~\ref{sec:ilang}, observations can be
viewed as merely ``actions taken by nature'' and vice versa.  At an extreme,
the planning problem and the plan may be identical:

\begin{lemma}[self-solving plans]
  If $P$ is a p-graph which is acyclic and the set of 
  its sink nodes is \mbox{$V_{\rm sink}$},
  then the plan $(P,V_{\rm sink})$ solves the planning problem
  $(P,V_{\rm sink})$.
\end{lemma}
\begin{proof}
  The plan is obviously finite and safe on itself. Because the set of
  joint-executions is simply the set of executions, the result follows from the
  fact that every execution on $P$ either reaches an element of
  $V_{\textrm{sink}}$, or is the prefix of one that does.
\end{proof}

We have described, in Definitions~\ref{defn:pairproduct} and \ref{defn:sd},
operations to construct new p-graphs out of old ones.  We can extend these in
natural ways to apply to plans.\footnote{\ldots and---via the symmetry between
Definitions~\ref{def:problem} and \ref{def:plan}---in the same stroke, to
planning problems, though in this paper we'll use these operations only on
plans.}

\begin{definition}[$\cup$-product of plans]
  The \emph{$\cup$-product} of two plans $(P, \Pterm)$ and $(Q,\Qterm)$, with $P$
  and $Q$ akin, is a plan $(P \uplus Q, \Pterm \cup \Qterm$.
\end{definition}

%IJRR: Fatemeh: include other things, like an intersection?
% Perhaps we can state that some operations are easy to do (like union)
% but the complexity comes at the cost of a complex definition of 
%  "solves"

%% The intuitive interpretation is that a plan which solves the $\cup$-product of two 
%% planning problems can guarantee achievement of some state satisfying the goal requirements from one
%% of the constituent planning problems --- even if uncertainty with respect to the 
%% initial conditions, or actions, or observations is reflected as non-determinism so
%% that the agent executing the plan cannot tell exactly which.
%% Note that a solution for
%% planning problem $P$ may or may not be a solution to a $\vee$-product of $P$
%% with some other plan since it must ensure finiteness on the pairwise product.

%% Definition~\ref{defn:ppproducts} can also be extended analogously to other
%% boolean expressions for the termination region. Perhaps the most useful is the
%% stronger $\cap$-product, where a solution reaches states satisfying goal
%% requirements from both constituent problems. This captures the idea of a plan
%% that is simultaneously a solution to two planning problems, even if the
%% solution need no know which planning problem is being solved.

\begin{theorem}[state-determined $\cup$-products] 
    Given plans $(P, \Pterm)$ and $(Q, \Qterm)$, with $P$ and $Q$ akin,
    construct a new plan whose p-graph, denoted $R$, is the expansion of $P
    \uplus Q$ into a state-determined presentation. 
    Recall that the expansion means that every state $s\in V(R)$ corresponds to
    sets $P_s\subseteq V(P)$ and $Q_s\subseteq V(Q)$ of states in the original
    p-graphs (either set can be empty, but never both).
    Define a termination region $\Rterm$ as follows:
    \vspace*{-2pt}
    {\small
        \[ \Rterm \defeq \left\{ s \in V(R) \;\left|\;  \left( P_s \neq \emptyset \wedge P_s\setminus\Pterm = \emptyset \right) \vee \left( Q_s \neq \emptyset \wedge  Q_s\setminus\Qterm = \emptyset \right) \right. \right\}.\vspace*{-2pt} \]
    }
    Then $(R, \Rterm)$ is equivalent to $(P \uplus Q, \Pterm \cup \Qterm)$
    in the sense of having identical sets of executions. 
    Moreover, any planning problem solved by the former is also solved by the
    latter.
  %\begin{enumerate}
  %  \item [$\wedge$:]% 
  %      if state $s\in V(P)$ is arrived at from $\eseq$, a sequence that results in
  %      $U_s\subseteq V(U)$ and $W_s\subseteq V(W)$ in the original p-graphs
  %      (either possibility empty, but not both), then $s \in \Vgoalpp$ iff
  %      $$ \left( \bigwedge_{v\in U_s} v\in V_0(U) \right) \wedge \left( \bigwedge_{v'\in W_s} v'\in V_0(W) \right);$$

  %  \item [$\vee$:]%
  %      if state $s\in V(P)$ is arrived at from $\eseq$, a sequence that results in
  %      $U_s\subseteq V(U)$ and $W_s\subseteq V(W)$ in the original p-graphs
  %      (either possibility empty, but not both), then $s \in \Vgoalpp$ iff
  %      $$ \left( \bigwedge_{v\in U_s} v\in V_0(U) \right) \vee \left( \bigwedge_{v'\in W_s} v'\in V_0(W) \right);$$

  %  \item [$\oplus$:]%
  %      if state $s\in V(P)$ is arrived at from $\eseq$, a sequence that results in
  %      $U_s\subseteq V(U)$ and $W_s\subseteq V(W)$ in the original p-graphs
  %      (either possibility empty, but not both), then $s \in \Vgoalpp$ iff
  %      $$ \left( \bigwedge_{v\in U_s} v\in V_0(U) \right) \oplus \left( \bigwedge_{v'\in W_s} v'\in V_0(W) \right).$$
  %\end{enumerate}
\end{theorem}
\begin{proof}
  This follows directly from the executions underlying the
  state-determined expansion, and the definition of the $\cup$-product.
\end{proof}

This result illustrates how the state-determined expansion is useful\,---\,it
permits a construction that captures the desired behavioral properties and, by
working from a standardized presentation, can do this directly by examining
states rather than posing questions quantified over the set of executions.

\subsection{Homomorphic solutions}
\label{sec:homomorophicsolns}

The following are a subclass of all solutions to a planning problem.

%\textcolor{blue}{
%\begin{definition}[quotient solution]\label{def:quotient}
%A plan $(P, \Vterm)$ is \emph{quotient} to planning problem $(W, \Vgoal)$, if and only if $\forall s_1, s_2\in \langof{P}\cap\langof{W}$,\newline  
%\[\operatorname{ReachVs}(s_1, {P})=
%  \operatorname{ReachVs}(s_2, {P})
%  \implies
%  \operatorname{ReachVs}(s_1, {W})=
%  \operatorname{ReachVs}(s_2, {W}),\]
%where $\operatorname{ReachVs}(s, {Q})$
%is the set of vertices reached tracing
%string $s$ on p-graph $Q$. 
%%\[\operatorname{RchVrt}(s, {Q}) \defeq \{ v \in \set{V(Q)}\;|\;v \text{ is reached by } s \text{ starting from some } v_0 \in \set{V_0(Q)}\}.\]
%\end{definition}
%}

\begin{definition}[homomorphic solution]\label{def:homomorphic}
  For a plan $(P, \Vterm)$ that solves planning problem $(W, \Vgoal)$, consider the
  relation $ R \subseteq V(P) \times V(W)$, in which $(v, w) \in R$ if and only
  if there exists a joint-execution on $P$ and $W$ that can end at $v$ in $P$
  and in $w$ in $W$.
  A plan for which this relation is a function is called a
  \emph{homomorphic solution}.
\end{definition}

% XXXXXXXXXXXXX
% --> This version causes problems with the claim that "label maps never induce homomorphism".
% \textcolor{blue}{
% \begin{definition}[state-determined homomorphic solution]\label{def:sthomomorphic}
%   For a plan $(P, \Vterm)$ that solves state-determined planning problem $(W, \Vgoal)$, consider the
%   relation $ R \subseteq V(P) \times V(W)$, in which $(v, w) \in R$ if and only
%   if there exists a joint-execution on $P$ and $W$ that can end at $v$ in $P$
%   and in $w$ in $W$.
%   %
%   A plan for which this relation is a function is called a
%   \emph{homomorphic solution}.
% \end{definition}
% \begin{definition}[homomorphic solution]\label{def:homomorphic}
%   A plan $(P, \Vterm)$ that solves planning problem $(W, \Vgoal)$ is
%   a \emph{homomorphic solution} 
% if $(P, \Vterm)$ is a homomorphic solution to $(W', \Vgoalp)$, 
% the state-determined presentation of $(W, \Vgoal)$.
% \end{definition}
% %
% }

%IJRR: Fatemeh has pulled up the standard definition of a homomorphism between
% graphs. It looks like two graphs that are homomorphic are p-graph homomorphic,
% but not necessarily the other way around. I.e., our definition includes a
% wider set. This would need to be stated and proved.

The name for this class of solutions comes via analogy to the
homomorphisms\,---that is, structure-preserving maps---\,which arise in algebra.  In
this context, a homomorphic solution is one for which each state in the plan
corresponds to exactly one state in the planning problem.

\begin{example}
  Recall Example~\ref{ex:35}, which shows a cyclic solution that involves
  tracing around the cyclic planning problem multiple times (until the least
  common multiple of their cycle lengths is found, in this case a series of 30
  states in each graph).
  This plan is \emph{not} a homomorphic solution because each plan state corresponds
  to multiple problem states.
  However, a simpler plan, depicted in Figure~\ref{fig:penta-plan}, can be
  formed in which each plan state maps to only one problem state.  This
  solution is therefore a homomorphic one.
  \begin{figure}[t]
    \centering
    \includegraphics[width=0.4\textwidth]{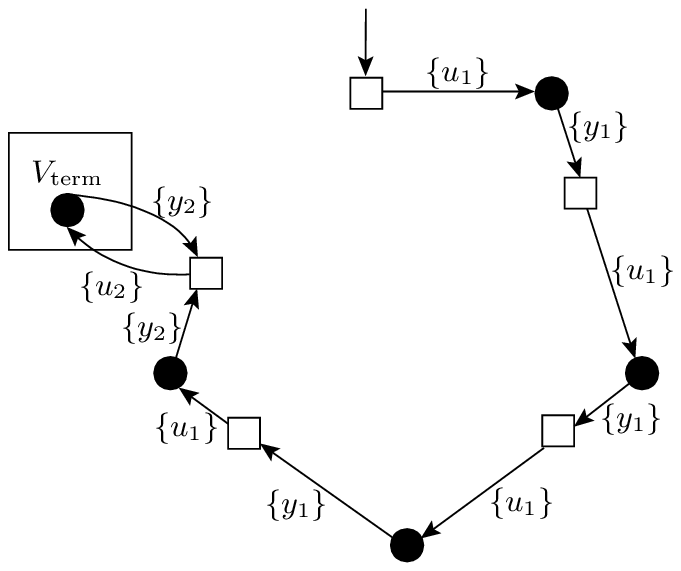}
    \caption{
      An alternative, more direct plan that solves the problem of navigating
      Figure~\ref{fig:pentaworld}'s robot to its charger.  This plan is a
      homomorphic solution.
      \label{fig:penta-plan}
    }
  \end{figure}
\end{example}

The preceding example is a particular instance of a more general pattern.

\begin{theorem}\label{thm:plan-implies-homplan}
  If there exists a plan to solve a \textcolor{black}{state-determined} planning problem, then there exists a
  homomorphic solution.
\end{theorem}

\begin{proof}
Suppose $(P, \Pterm)$ is a solution to $(W, \Vgoal)$%
% \textcolor{blue}{
% , assuming (without loss of generality) the latter to be state-determined.
% }
.
If every joint-execution arriving
at $v$ in $P$ arrives at the same $w$ in $W$, then $(v,w)\in R$ is a function, so
$(P, \Pterm)$ is a homomorphic solution. 
Thus, consider the cases for which there are elements $(v,w)\in R$ and
$(v,y)\in R$, with $w \neq y$. 
Let $R_\textrm{last} \subset R$ be the relation where $(v_p, v_w) \in R_\textrm{last}$ iff there is a
joint-execution $\eseq$ arriving at 
$v_p$ on $P$  and $v_w$ on $W$, and there are no joint-executions which 
extend the execution (e.g., $e_1 \cdots e_k \cdots e_m$, $m>k$) that arrive at $v_w$ again.
Then construct a new plan $(Q, \Qterm)$ with $V(Q) = V(W)$ and  $V_0(Q) =
V_0(W)$. 
For all edges departing $v\in P$
associated with $w\in Q$ where $(v, w) \in R_\textrm{last}$, we collect the 
label sets by unioning them to form $V_e$.
% TODO: Associated or in $w\in V(Q)$, and probably clearer to have 
% $w\in V(W)$, since W = Q
Then edges departing $w$ are included in $Q$ by carrying over edges from $W$,
intersecting $V_e$ with all the labels of edges departing $w$, and dropping those
for which the result is empty.
Finally, an element $w$ is included in $\Qterm$ if there is a 
$v\in \Pterm$ with $(v, w) \in R_\textrm{last}$.
Then $(Q, \Qterm)$ is a solution to $(W, \Vgoal)$ because, though 
$(P, \Pterm)$ and $(Q, \Qterm)$ have different sets of executions, every 
execution on $P$ that reaches $\Pterm$ is transformed into another on
$Q$ reaching $\Qterm$ (and $\Vgoal$).
Moreover, this
ensures that the relation from Definition~\ref{def:homomorphic} is a bijection, so that $(Q, \Qterm)$ is a homomorphic
solution to $(W, \Vgoal)$.
\end{proof}

\subsection{Destructive or not?}

If a label map can express a change in a p-graph, the question is whether
this change matters. One can pose this question meaningfully for planning
problems as the added ingredients provide semantics that yield the notion of
solubility. 

\begin{definition}[destructive and non-destructive on plans]
  A label map $h$ is \emph{destructive} on a set of solutions $S$ to planning
  problem $(G,\Vgoal)$ if, for every plan $(P,\Vterm)\in S$, $(h(P),\Vterm)$
  cannot solve $(h(G), \Vgoal)$.  
  We say that $h$ is \emph{non-destructive} on $S$ if for every 
  $(P,\Vterm)\in S$, the plan $(h(P),\Vterm)$ does solve $(h(G), \Vgoal)$.  
  \label{defn:nondestructive}
\end{definition}

Intuitively, destructiveness requires that the label map break all
existing solutions; non-destructiveness requires that the label map
break none of them.

\begin{example}[single plans]
  If $S=\{s\}$ is a singleton set, then we can determine whether $h$ is
  destructive on $S$ by applying the label map $h$ ---recall
  Definition~\ref{def:set-map}--- to compute $h(s)$ and $h(G)$, and then
  testing whether $h(s)$ solves $h(G)$ ---recall the algorithm described in
  Section~\ref{sec:plans}.  If $h(s)$ solves $h(G)$, then $h$ is nondestructive on
  $S$; otherwise, $h$ is destructive on $S$.
  In this singleton case, we say simply that $h$ is (non-)destructive on $s$.
\end{example}
\vspace*{-2pt}

Definition~\ref{defn:nondestructive} depends on a selection of some class of
solutions.  Of particular interest is the maximal case, in which every solution
is part of the class.

\vspace*{-2pt}
\begin{definition}[strongly destructive and strongly non-destructive]
  A label map $h$ is \emph{strongly (non-)destructive} on a planning problem
  $(G, \Vgoal)$ if it is \mbox{(non-)destructive} on the set of all solutions to $(G,
  \Vgoal)$.  \label{defn:nondestructive2}
\end{definition}
\vspace*{-2pt}

Note that, while strong destructiveness may be decided by attempting to
generate a plan for $h(G)$ (perhaps by backchaining from $\Vgoal$), strong
non-destructiveness may be quite difficult to verify in general, if only due to
the sheer variety of extant solutions.  (Recall Example~\ref{ex:35}, which
solves its problem in an unexpected way.)  The next results, while not
sufficient in general to decide whether a map is strongly non-destructive, do
perhaps shed some light on how that might be accomplished.
% provide some hope that it may be possible.

\vspace*{-3pt}
\begin{lemma}[label maps preserve safety]\label{lem:map-safe}
  If $P$ is safe on $G$, then for any label map $h$, $h(P)$ is safe on $h(G)$.
\end{lemma}
\vspace*{-3pt}
\begin{proof}
  Consider each pair of states $(v, w)$, with $v \in V(P)$ and $w \in V(G)$
  reached by some joint-execution on $P$ and $G$.
  Suppose for simplicity that $v$ is an action state.  (The observation case is
  similar.)
  Let $E_1$ denote the union of all labels for edges outgoing from $v$, and
  likewise $E_2$ for labels of edges outgoing from $w$.  Since $P$ is safe on
  $G$, we have $E_1 \subseteq E_2$.
  Then, in $h(P)$ and $h(G)$, observe that

  \vspace{-8pt}
    $$h(E_1) = \bigcup_{e \in E_1} h(e) \subseteq \bigcup_{e \in E_2} h(e) = h(E_2), $$
  \vspace{-8pt}

  \noindent and conclude that $h(P)$ is safe on $h(G)$.
\end{proof}

% \textcolor{red}{
% The following lemma is no longer true under our updated definition of
% homomorphism: I can give a simple example of a label map that 
% turns a non-homomorphic plan into one that is.
% }

\begin{lemma}[label maps never introduce homomorphism]\label{lem:maps-no-homo}
  If $(P,\Pterm)$ is a non-homomorphic solution to $(G,\Vgoal)$ then 
  no label map $h$ results in $(h(P),\Pterm)$ being a homomorphic solution
  to $(h(G),\Vgoal)$.
\end{lemma}
\begin{proof}
Since $(P,\Pterm)$ is a non-homomorphic solution to $(G,\Vgoal)$, there exist
two joint-executions $\eseq$ and $\epseq$ on $P$ and $G$ such that both arrive
at $v\in V(P)$ in $P$, but on $G$, the former arrives at $w\in V(G)$ and the
latter arrives at $w'\in V(G)$ with $w \neq w'$.  Now, given any $h(\cdot)$,
pick any particular sequence ${(h_1\in h(e_1)) \cdots (h_k\in h(e_k))}$, and
${(h'_1\in h(e'_1)) \cdots (h'_m\in h(e'_m))}$, making choices arbitrarily.
These are joint-executions on $h(P)$ and $h(G)$.  Application of the label
map means there is a way of tracing both ${(h_1\in h(e_1)) \cdots (h_k\in
h(e_k))}$ and ${(h'_1\in h(e'_1)) \cdots (h'_m\in h(e'_m))}$ on $h(P)$ to arrive
at $v$, while there is a way of tracing the former on $h(G)$ to arrive at $w$,
and the latter at $w'$.  
% \textcolor{red}{
% The problem is that doing the state-determined expansion of $h(G)$ can merge
% $w$ and $w'$ into one state.
% }
So $(h(P),\Pterm)$ cannot be a homomorphic solution to
$(h(G),\Vgoal)$.  \end{proof}

\newcommand{\allhom}{\mathcal{H}}

\begin{theorem}[extensive destructiveness]\label{thm:ed}
  For a \textcolor{black}{state-determined} planning problem $(G, \Vgoal)$, let
  $\allhom$ denote the set of homomorphic solutions that problem.
  Then any label map that is destructive on $\allhom$ is strongly destructive.
\end{theorem}
\begin{proof}
  Since $h$ is destructive on $\allhom$, we know that $(h(G), \Vgoal)$ can only
  have homomorphic solutions if some formerly non-homomorphic solution can
  become a homomorphic one under $h$, but Lemma~\ref{lem:maps-no-homo}
  precludes that eventuality.
  This implies, via Theorem~\ref{thm:plan-implies-homplan}, that \emph{no} plan
  solves $h(G)$.  Therefore $h$ is strongly destructive on $(G, \Vgoal)$.
\end{proof}

The interesting thing here is that Theorem~\ref{thm:ed} shows that the class of
homomorphic solutions play a special role in the space of all plans: By
examining the behavior of $h$ on $\allhom$, we can gain some insight into its
behavior on the space of all plans.  Informally, $\allhom$ seems to function
as a `kernel' of the space of all plans.

\section{Related work}\label{sec:related}

In earlier sections of the paper we have interspersed precise connections to
specific prior work. This section supplements those links by taking a wider
view; the purpose is not merely coverage, but rather broader context. 

\smallskip

\begin{comment}
In the broadest of terms, our work echoes \citet{plato}, whose ``Allegory of the Cave'', relates the realm of ideals to those of
corporeal experience by talking of the former being cast as shadows for beings
to see.  This {\em projective} relationship captures closely the relationship
between the mathematical object we intend to be useful for algorithm design,
and the model of the physical reification. Sensor transformations, like
$\func{h_{\textrm{rose}}}$ (in Example~\ref{ex:rose} on pg.~\pageref{ex:rose}),
lose information by identifying certain subsets of events.  The question we ask
is whether the behavior of the robot is equivalent modulo this congruence.  The
methodology supposes that the idealized sensor captures all the possible
structured aspects of a sensor because transformations, in modeling hardware
realizations, may only lose structure. The model represents that loss
information, after some indirection, as non-determinism.  Non-determinism is a
very weak position, for it says nothing about regularity that may exist in the
ways sensors are non-ideal. 
\end{comment}

This work builds most directly on, and is strongly influenced by, the
combinatorial filtering perspective, with its use of simple, discrete objects
that generalize beyond the methods used in traditional estimation theory, which
has a strong reliance on probabilistic models.  A gap still remains between the
theory of discrete combinatorial filters and the probabilistic, typically
recursive Bayes formulations, employed most often in practice on robots today.
Both types have a long history.  The probabilistic filters go back
to~\citet{kalman60}, having found use in several important problems in mobile
robotics, including estimation of robot pose and map
information~\citep{dissanayake01,smith90}.  This class of filters is well-known
within the community, with a vast surge of interest catalyzed by the
publication of the book by~\citet{thrun05}.  The discrete filters we focus on
in this paper have their roots in the minimalist manipulation work of
\citet{ErdMas88} and \citet{Gol93}.  They were formalized more generally by
\citet{Lav06}, though this paper evolves those models in a new direction.

Discrete filters and their related sorts of representations are recherch\'{e}
rather than simply obscure: They have been employed in the form of combinatorial
filters to successfully solve a wide a range of useful tasks; recent examples
include target tracking~\citep{yu12shadow}, mobile robot
navigation~\citep{lopez12optimal,tovar07gnt}, and
manipulation~\citep{kristek12}.  Both \citet{lavalle12sensing}, which
provides a tutorial introduction and overview to the approach, and the
substantial paper on the topic \citet{TovCoh+14}, recognize that more work is
needed to extend the theory.  There are two directions which have demanded attention.
The first, which the authors of both of the preceding papers identify, is
that, thus far, the approach provides only for inference and more work is
needed in order to express aspects of feedback-aware control for achieving
tasks.  Section~\ref{sec:plan} has begun to address this gap. 

The second direction is born of the observation that all the combinatorial
filters in the existing work deal with extremely simple sensors. How might
combinatorial filters and other discrete models scale up to larger problems?
One may quite rightly criticize such filters on the basis of their size or 
expression complexity.  An important contribution of the present paper is
increasing in the complexity of sensors that may be treated by discrete filters
without necessitating an enormous growth of filter size.  Previously, when sets
of observations were treated, they required duplication (usually of an edge in
a graph structure), causing substantial blow-up of the model.  The form of
filters we examine does not assume that the set of possible observations is
finite.  And we have described results for filters with infinite (though
finitely described) subsets of $\real$ as labels.  This idea was inspired by
\citet{veanes12symbolictransducers}, who developed symbolic finite transducers
that are concise and expressive for processing strings over large alphabets.
It is also worth nothing that a set of techniques have been developed, along
quite separate lines, to reduce or simplify the representation of information
within such filters~\citep{OKa11,OKaShe17,SonOKa12}. 

%Within control systems literature, the term {\em
%observability}\,\citep{stefano95feedback} is related to the problem we study
%but is, classically at least, concerned with systems that are continuous and
%linear. In those cases,  computing the observability index is essentially a
%rank computation of a matrix constructed by stacking products of the output
%measurement model matrix and (powers of) the plant dynamics matrix.  Although
%we consider label sets which may include subsets of $\real$, the mapping which
%defines a filter is seldom expected to be a linear function of the input and,
%as will become clear (see, especially, Section~\ref{sec:hardness}), the problem
%addressed herein is far more challenging computationally.

Both directions have demanded generalization in slightly different ways.
We
believe that one of the most useful aspects of the formalism arising from this
generalization has been the notion of label maps.  These functions allow one
to degrade models, starting (as we did in the iRobot Create case study) with
physically unrealizable idealizations, and gradually exploring how behavior is
altered.
There is, in fact, a long history and existing precedent for
studying intelligent systems under sensor perturbations.  Psychologist George
Stratton (shown in Figure~\ref{fig:stratton}) pioneered the study of perception
in human vision by having subjects wear special glasses that inverted
images~\cite{Str97}.
Stratton observed that after a relatively short adaptation
period, the subjects began to perceive the world normally, in spite of the
vertical inversion.
This is effectively a sensor map $(x,y) \mapsto (x,-y)$ for suitably
chosen coordinates.
It is a continuous transformation, satisfying the monotonicity requirement
identified in Example~\ref{example:monot}, and ---as Stratton observed--- the
map is not destructive.

\begin{figure}
  \centering
  \includegraphics[scale=0.2]{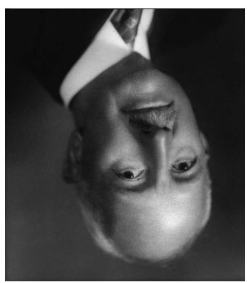}
  \caption{George Stratton, as he would have appeared in the first few minutes
  of wearing the inversion glasses he pioneered.
  ``If a subject is made to wear
  glasses which correct the retinal images, the whole landscape at first
  appears unreal and upside down; on the second day of the experiment normal
  perception begins to reassert itself\ldots'' \citep[p.285]{merleauPonty62} }
  \label{fig:stratton}
\end{figure}

% By definition, theoretical idealizations are simplifications. But
% what simplifications are faithful for a class of devices? Which are useful? A
% hierarchy of maps allows one to get direct purchase on this problem, a problem
% which is rightly called foundational--determining the right models for (and
% abstractions of) the device being studied underlies the entire engineering
% enterprise.  How this is represented within in the formalism and its
% connections are surprising.

One recent formulation that emphasizes action from the outset is Erdmann's more
recent work on strategy complexes~\citep{erdmann12topology,erdmann10topology},
as referenced in Example~\ref{ex:erdmann}.  He uses tools from classic and
computational topology to relate plans, formulated broadly to include sources
of non-determinism, to high-dimensional objects---his loopback
complexes---whose homotopy type provides information about whether the planning
problem can be solved.  We speculate that preservation of plan existence under
label maps might be productively studied across planning problems by examining
the map's operation on loopback complexes: classes of maps that can be shown to
preserve the homotopy type of such complexes (perhaps over restricted classes
of planning problems) can be declared non-destructive.  

An alternative approach, with goals similar to our own ---namely of identifying
representational basis for objects that can manipulated by algorithms in order
to guide the design process--- is due to \citet{censi17co}.  He poses and
solves co-design problems; ascertaining the maximal task set achievable for a
given set of resources.  He shows that, given a network of monotone
constraints, the selection of components is a process that can be efficiently
automated.  Part of the present interest in studying labels maps is that they
can model aspects of different components.

Also adopting an algorithmic stance on the design process, are methods based on
hybrid automata, which blend discrete and continuous elements.  Powerful
synthesis and verification techniques are known for these
models~\cite{belta07,RamPit+15,DecKre16}. Despite some similarities, including
extensive use of non-determinism, the relationship between p-graphs and hybrid
automata is somewhat involved: guard expressions in a rich logical
specification language have structure missing from the label sets we study; the
action labels in p-graphs are not intended to model continuous dynamics.

\textcolor{black}{
% connections to process algebra, bisimulation, papers by Tabuada, category theory?
Though the present paper has focused on generalization and idealization to a degree
perhaps uncommon in the robotics literature, this abstract style of approach
in fundamental treatments of behavior appears in other settings.  The natural question
is how these treatments are related.  At least for the question of
bisimulation, a notion of equivalence employed in process algebras (and,
importantly with respect to the present study, along with some generalizations
to systems with continuous dynamics,
see~\cite{haghverdi05}), one of the authors has recently obtained a clear
result on the relationship between the bisimulation relation and filter
reduction. \cite{RahmOKa18} show that filter reduction  (see
Section~\ref{sec:hardness}) can be achieved by quotienting an input filter by
some relation and that bisimilarity is not the correct
notion of equivalence for some types of filters. 
}

\section{Conclusion}\label{sec:conclusion}
This paper introduced and explored formalisms for reasoning about interactions
between robots and their environments, including interaction languages,
p-graphs, and label maps.
We believe that the most crucial intellectual contributions of the present work
are in attaining a degree abstractness missing from prior ideas in two ways.

First, we separate those entities which have been formalized in robotics because
they have some interpretation that is useful (e.g., the idea of a plan, a
filter), from their representation.  The p-graph, in and of itself, lacks an
obvious interpretation.  Its definition does not include semantics belying a
single anticipated use, rather context and any specific interpretation are only
added for the special subclasses. In this sense, it is identical to the abstract
treatment of computation as the constructive process of realizing a
correspondence from inputs to outputs.

Second, even if something like a p-graph is a representation that is general
enough to express many items of interest, it is not a {\em
canonical form}.  This paper engenders an important mental shift in lifting
most of the notions of equivalence up to sets of executions, via interaction
languages, rather than depending on operations on some specific graph.  The
present work continues to separate the notion of behavior from presentation.
This helps establish a foundation for the semantics of the coupled
robot-environment system.  

%Despite the theoretical form of this paper, it describes
%work motivated by deeply practical questions.  More work remains to be done to
%continue casting the problems robotics and interests of practitioners in
%expressions with this broader spirit.

The theoretical groundwork laid by this paper for reasoning about sensors and
actuators, and their associated estimation and planning processes, aims to
strengthen the link between idealized models and practical systems.  It is
imperative that we close the gap between robotics science and robotics
practice, and more work remains to be done. We submit that it should be work aimed
unambiguously and explicitly at that gap.

\section*{Acknowledgements} 
This material is based upon work supported by the National Science Foundation
under Grants IIS-1527436, IIS-1526862, IIS-0953503, IIS-1453652.  We thank
Yulin Zhang for his close reading of earlier versions of this work.

\bibliographystyle{SageH}
\bibliography{bibijrr}

\end{document}